\documentclass[twoside,11pt]{article}

\usepackage[abbrvbib]{jmlr2e_arxiv}

\usepackage{amsmath}
\usepackage{bbm}
\usepackage{algorithm}
\usepackage{algorithmic}
\usepackage{enumitem}
\usepackage{multirow}
\usepackage{subcaption}
\usepackage{lastpage}

\newcommand\MTR{IWR}

\def\eg{\emph{e.g.}}
\def\ie{\emph{i.e.}}

\def\R{{\mathbb R}}

\DeclareMathOperator{\E}{\mathbb{E}}
\DeclareMathOperator{\1}{\mathbbm{1}}

\graphicspath{{figures/}}

\title{A Contextual Bandit Bake-off}

\author{\name Alberto Bietti \email alberto.bietti@nyu.edu \\
    \addr Center for Data Science, New York University, New York, NY\thanks{Work done while AB was at Inria, partly during a visit to Microsoft Research NY, supported by the Microsoft Research-Inria Joint Center.} \\
    \name Alekh Agarwal \email alekha@microsoft.com \\
    \addr Microsoft Research, Redmond, WA \\
    \name John Langford \email jcl@microsoft.com \\
    \addr Microsoft Research, New York, NY
}
\jmlrheading{22}{2021}{1-\pageref{LastPage}}{12/18; Revised 1/21}{5/21}{18-863}{Alberto Bietti, Alekh Agarwal and John Langford}
\ShortHeadings{A Contextual Bandit Bake-off}{Bietti, Agarwal and Langford}
\firstpageno{1}

\begin{document}

\maketitle

\begin{abstract}
Contextual bandit algorithms are essential for solving many real-world
interactive machine learning problems. Despite multiple recent successes
on statistically optimal and computationally efficient methods, the practical
behavior of these algorithms is still poorly understood.
We leverage the availability of large numbers of supervised learning datasets to
empirically evaluate contextual bandit algorithms,
focusing on practical methods that learn by relying on optimization oracles
from supervised learning. We find that a recent method~\citep{foster2018practical} using optimism under uncertainty works the best overall.
A surprisingly close second is a simple greedy baseline that only explores implicitly through the diversity of contexts, followed by a variant of Online Cover~\citep{agarwal2014taming} which tends to be more conservative but robust to problem specification by design.
Along the way, we also evaluate various components of contextual bandit algorithm design such as loss estimators.
Overall, this is a thorough study and review of contextual bandit methodology.
\end{abstract}

\section{Introduction} 
\label{sec:introduction}

At a practical level, how should contextual bandit learning and
exploration be done?

In the contextual bandit problem, a learner repeatedly observes a
context, chooses an action, and observes a loss for the chosen action
only.  Many real-world interactive machine learning tasks are
well-suited to this setting: a movie recommendation system selects a
movie for a given user and receives feedback (click or no
click) only for that movie; a choice of medical treatment may be
prescribed to a patient with an outcome observed for (only) the chosen
treatment.  The limited feedback (known as \emph{bandit} feedback)
received by the learner highlights the importance
of \emph{exploration}, which needs to be addressed by contextual
bandit algorithms.

The focal point of contextual bandit (henceforth CB) learning research is efficient
exploration algorithms~\citep{abbasi2011improved,agarwal2012contextual,agarwal2014taming,
agrawal2013thompson,dudik2011efficient,langford2008epoch,russo2017tutorial}.
However, many of these algorithms remain far from practical, and even when considering more
practical variants, their empirical behavior is poorly understood, typically with
limited evaluation on just a handful of scenarios.
In particular, strategies based on upper confidence bounds~\citep{abbasi2011improved,li2010contextual}
or Thompson sampling~\citep{agrawal2013thompson,russo2017tutorial}
are often intractable for sparse, high-dimensional datasets,
and make strong assumptions on the model representation.
The method of~\citet{agarwal2014taming} alleviates some of these difficulties while being statistically optimal
under weak assumptions,
but the analyzed version is still far from practical, and the worst-case guarantees may lead
to overly conservative exploration that can be inefficient in practice.

The main objective of our work is an evaluation of practical
methods that are relevant to practitioners.
We focus on algorithms that rely on \emph{optimization oracles} from supervised learning 
such as cost-sensitive classification or regression oracles,
which provides computational efficiency and support for generic representations.
We further rely on online learning implementations of the oracles,
which are desirable in practice due to the sequential nature of contextual bandits.
While confidence-based strategies and Thompson sampling are not directly adapted to this setting,
we achieve it with online Bootstrap approximations for Thompson sampling~\citep{agarwal2014taming,eckles2014thompson,osband2015bootstrapped},
and with the confidence-based method of~\citet{foster2018practical} based on regression oracles,
which contains LinUCB as a special case.
Additionally, we consider practical design choices such as loss encodings
(\eg, if values have a range of 1, should we encode the costs of best and worst outcomes as 0/1 or -1/0?),
and for methods that learn by reduction to off-policy learning,
we study different reduction techniques beyond the simple inverse propensity scoring approach.
All of our experiments are based on the online learning system Vowpal
Wabbit\footnote{\url{https://vowpalwabbit.org}} which has already been
successfully used in production systems~\citep{agarwal2016multiworld}.

The interactive aspect of CB problems makes them notoriously difficult
to evaluate in real-world settings beyond a handful of tasks.
Instead, we leverage the wide availability of supervised learning datasets
with different cost structures on their predictions, and obtain contextual bandit instances
by simulating bandit feedback, treating labels as actions and hiding the loss of all actions
but the chosen one.
This setup captures the generality of the i.i.d.~contextual bandit setting, while avoiding some
difficult aspects of real-world settings that are not supported by most existing algorithms
and are difficult to evaluate, such as non-stationarity.
We consider a large collection of over 500 datasets with varying characteristics and various cost structures,
including multiclass, multilabel and more general cost-sensitive datasets with real-valued costs.
To our knowledge, this is the first evaluation of contextual bandit algorithms
on such a large and diverse corpus of datasets.

Our evaluation considers online implementations of Bootstrap Thompson
sampling~\citep{agarwal2014taming,eckles2014thompson,osband2015bootstrapped},
the Cover approach of~\citet{agarwal2014taming},
$\epsilon$-greedy~\citep{langford2008epoch},
RegCB~\citep[which includes LinUCB as a special
case]{foster2018practical}, and a basic greedy method similar to the one studied
in~\citet{bastani2017exploiting,kannan2018smoothed}.
As the first conclusion of our study, we find that the recent RegCB method~\citep{foster2018practical} performs the best overall across a number of experimental conditions.
Remarkably, we discover that a close second in our set of methods is the simple greedy baseline, often outperforming most exploration algorithms.
Both these methods have drawbacks in theory; greedy can fail arbitrarily poorly in problems where intentional exploration matters, while UCB methods make stronger modeling assumptions and can have an uncontrolled regret when the assumptions fail. The logs collected by deploying these methods in practice are also unfit for later off-policy experiments, an important practical consideration.
Our third conclusion is that several methods which are more robust in that they make only a relatively milder
i.i.d.~assumption on the problem tend to be overly conservative and often pay a steep price on easier datasets.
Nevertheless, we find that an adaptation of Online Cover~\citep{agarwal2014taming} is quite competitive on a large fraction of our datasets.
We also evaluate the effect of different ways to encode losses and study different reduction mechanisms for
exploration algorithms that rely on off-policy learning (such as~$\epsilon$-greedy),
finding that a technique based on importance-weighted regression tends to outperform other approaches when
applicable.
We show pairwise comparisons between the top 3 methods in our evaluation in Figure~\ref{fig:greed_is_good} for datasets with 5 or more actions.
For future theoretical research, our results motivate an emphasis on understanding greedy strategies, building on recent progress~\citep{bastani2017exploiting,kannan2018smoothed}, as well as effectively leveraging easier datasets in exploration problems~\citep{agarwal2017open}.

\begin{figure}[tb]
	\begin{center}
	\includegraphics[width=0.30\columnwidth]{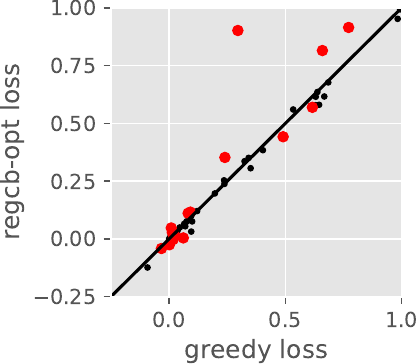}
	\includegraphics[width=0.30\columnwidth]{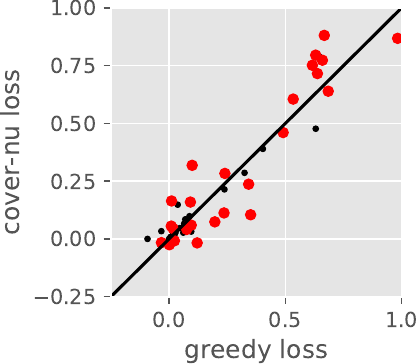}
	\includegraphics[width=0.30\columnwidth]{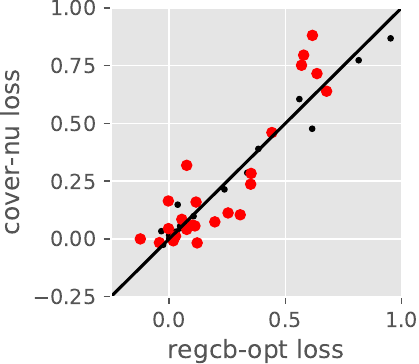}
	\end{center}
\caption{Comparison between three competitive approaches: RegCB (confidence based),
Cover-NU (variant of Online Cover) and Greedy.
The plots show relative loss compared to supervised learning (lower is better) on all datasets with 5 actions or more.
Red points indicate datasets with a statistically significant difference in loss between two methods.
A greedy approach can outperform exploration methods in many cases; yet both Greedy and RegCB may
fail to explore efficiently on some other datasets where Cover-NU dominates.}
	\label{fig:greed_is_good}
\end{figure}

\subsection{Organization of the paper}
The paper is organized as follows:
\begin{itemize}[noitemsep,topsep=1pt]
	\item Section~\ref{sec:setup} provides relevant background on i.i.d.~contextual bandits, optimization
	oracles, and mechanisms for reduction to off-policy learning, and introduces our experimental setup.
	\item Section~\ref{sec:algorithms} describes the main algorithms we consider in our evaluation,
	as well as the modifications that we found effective in practice.
	\item Section~\ref{sec:experiments} presents the results and insights from our experimental evaluation.
	\item Finally, we conclude in Section~\ref{sec:discussion} with a discussion of our findings
	and a collection of guidelines and recommendations for practitioners that come out
	of our empirical study, as well as open questions for theoreticians.
\end{itemize}

\section{Contextual Bandit Setup}
\label{sec:setup}

In this section, we present the learning setup considered in this work,
recalling the stochastic contextual bandit setting, the notion of optimization
oracles,
various techniques used by contextual bandit algorithms for leveraging these oracles,
and finally our experimental setup.

\subsection{Learning Setting}

The stochastic~(i.i.d.) contextual bandit learning problem can be described as follows.
At each time step~$t$, the environment produces a pair~$(x_t, \ell_t) \sim D$
independently from the past, where~$x_t \in \mathcal{X}$ is a context vector and~$\ell_t = (\ell_t(1), \ldots, \ell_t(K)) \in \R^K$ is a loss vector, with~$K$ the number of possible actions,
and the data distribution is denoted~$D$.
After observing the context~$x_t$, the learner chooses an action~$a_t$, and only observes the loss~$\ell_t(a_t)$ corresponding to the chosen action.
The goal of the learner is to trade-off exploration and exploitation in order to incur a small cumulative regret
\begin{equation*}
R_T := \sum_{t=1}^T \ell_t(a_t) - \sum_{t=1}^T \ell_t(\pi^*(x_t)),
\end{equation*}
with respect to the optimal policy~$\pi^* \in \arg\min_{\pi \in \Pi} \E_{(x, \ell) \sim D}[\ell(\pi(x))]$,
where~$\Pi$ denotes a (large, possibly infinite) set of policies~$\pi : \mathcal{X} \to \{1, \ldots, K\}$
which we would like to do well against.
It is often important for the learner to use randomized strategies,
for instance in order to later evaluate or optimize new policies,
hence we let~$p_t(a) \in [0,1]$ denote the probability that the agent chooses action~$a \in \{1, \ldots, K\}$ at time~$t$, so that~$a_t\sim p_t$.

\subsection{Optimization Oracles}
\label{sub:oracles}

In this paper, we focus on CB algorithms which rely on access to an \emph{optimization oracle}
for solving optimization problems similar to those that arise in supervised learning,
leading to methods that are suitable for general policy classes~$\Pi$.
The main example is the \textbf{cost-sensitive classification (CSC) oracle}~\citep{agarwal2014taming,dudik2011efficient,langford2008epoch},
which given a collection $\{(x_t, c_t)\}_{t=1,\ldots,T} \subset \mathcal{X} \times \R^K$ computes
\begin{equation}
\label{eq:csc_oracle}
\arg\min_{\pi\in \Pi} \sum_{t=1}^T c_t(\pi(x_t)).
\end{equation}
The cost vectors~$c_t = (c_t(1), \ldots, c_t(K)) \in \R^K$ are often constructed using counterfactual estimates of the true (unobserved) losses, as we describe in the next section.

Another approach is to use \textbf{regression oracles}, which find~$f : \mathcal{X} \times \{1, \ldots, K\} \to \R$
from a class of regressor functions~$\mathcal{F}$
to predict a cost~$y_t$, given a context~$x_t$ and action~$a_t$~\citep[see, \eg,][]{agarwal2012contextual,foster2018practical}.
In this paper, we consider the following regression oracle with importance weights $\omega_t > 0$, which given a collection~$\{(x_t, a_t, y_t, \omega_t)\}_{t=1,\ldots,T}$ computes
\begin{equation}
\label{eq:reg_oracle}
\arg\min_{f\in \mathcal{F}} \sum_{t=1}^T \omega_t(f(x_t, a_t) - y_t)^2.
\end{equation}

While the theory typically requires exact solutions to~\eqref{eq:csc_oracle} or~\eqref{eq:reg_oracle},
this is often impractical due to the difficulty of the underlying optimization problem (especially for CSC, which yields a non-convex and non-smooth problem),
and more importantly because the size of the problems to be solved keeps increasing after each iteration. In this work, we consider instead the use of \emph{online optimization oracles} for solving problems~\eqref{eq:csc_oracle}
or~\eqref{eq:reg_oracle}, which incrementally update a given policy or regression function after each new observation,
using for instance an online gradient method.
Such an online learning approach is natural in the CB setting,
and is common in interactive production systems~\citep[\eg,][]{agarwal2016multiworld,he2014practical,mcmahan2013ad}.
More details on the implementation of these online oracles are given in Appendix~\ref{sec:alg_details}.

\subsection{Loss Estimates and Reductions}
\label{sub:reductions}
A common approach to solving problems with bandit (partial) feedback is to compute an estimate of the full feedback using the observed loss and then apply methods for the full-information setting to these estimated values.
With enough randomization in the choices of actions, these can provide good \emph{counterfactual} estimates for other actions, and in turn, for new policies.
In the case of CBs, this allows an algorithm to find a ``good'' policy based on
\emph{off-policy} exploration data previously collected by the algorithm,
for instance through various forms of off-policy learning~\citep[see, \eg,][]{dudik2011doubly}.

Given loss estimates~$\hat \ell_t(a)$ of~$\ell_t(a)$ for all actions~$a$ (including unobserved ones) and for~$t = 1, \ldots, T$,
one may directly write such an off-policy learning problem as optimizing a CSC objective of the form~\eqref{eq:csc_oracle} over policies, where~$x_t$ are the observed contexts and the cost vectors are defined by~$c_t = (\hat \ell_t(1), \ldots, \hat \ell_t(K)) \in \R^K$.
A CB algorithm that relies on policies trained in this manner may be said to operate by \emph{reduction to off-policy learning}. A canonical example of this is the~$\epsilon$-Greedy algorithm~\citep{langford2008epoch}.
More generally, CB algorithms may use these loss estimates to create more complex cost vectors~$c_t$ which could include additional bonuses for certain actions~\citep[\eg,][]{agarwal2014taming,dudik2011efficient},
a process which we refer to as \emph{reduction to cost-sensitive classification}, since this does not directly correspond to off-policy learning.

We now describe the three different mechanisms considered in this paper for loss estimation and reduction to off-policy learning. The first two compute vectors~$\hat \ell_t \in \R^K$ of loss estimates; these may then be fed directly to a CSC oracle for off-policy learning, or used as part of more general reductions to CSC. The third approach relies directly on a regression oracle for reducing to off-policy learning.
In what follows, we consider observed interaction records~$(x_t, a_t, \ell_t(a_t), p_t(a_t))$.

\subparagraph{IPS (inverse propensity-scoring).}
Perhaps the simplest approach is the following IPS estimator:
\begin{equation}
\label{eq:ips}
\hat{\ell}_t(a) := \frac{\ell_t(a_t)}{p_t(a_t)} \1\{a = a_t\}.
\end{equation}
For any action~$a$ with $p_t(a) > 0$, this estimator is unbiased, \ie~$\E_{a_t\sim p_t}[\hat{\ell}_t(a)] = \ell_t(a)$, but can have high variance when~$p_t(a_t)$ is small.
The estimator leads to a straightforward CSC example~$(x_t, \hat{\ell}_t)$.
Using such examples in~\eqref{eq:csc_oracle} provides a way to perform off-policy (or counterfactual) evaluation and optimization, which in turn allows a CB algorithm to identify good policies for exploration.
In order to obtain good unbiased estimates, one needs to control the variance of the estimates, \eg, by enforcing a minimum exploration probability~$p_t(a) \geq \epsilon > 0$ on all actions.

\subparagraph{DR (doubly robust).}
In order to reduce the variance of IPS, the doubly robust estimator~\citep{dudik2011doubly} uses a separate, possibly biased, estimator of the loss $\hat{\ell}(x, a)$:
\begin{equation}
\label{eq:dr}
\hat{\ell}_t(a) := \frac{\ell_t(a_t) - \hat{\ell}(x_t, a_t)}{p_t(a_t)} \1\{a = a_t\} + \hat{\ell}(x_t, a).
\end{equation}
When $\hat{\ell}(x_t, a_t)$ is a good estimate of~$\ell_t(a_t)$, the small numerator in the first term helps reduce the variance induced by a small denominator, while the second term ensures that the estimator is unbiased on the support of~$p_t$.
Typically, $\hat{\ell}(x, a)$ is learned by regression on all past observed losses, \eg,
\begin{equation}
\label{eq:loss_estimator_def}
\hat \ell := \arg\min_{f \in \mathcal F} \sum_{t' \leq t} (f(x_{t'}, a_{t'}) - \ell_{t'}(a_{t'}))^2.
\end{equation}
The reduction to cost-sensitive classification is similar to IPS, by feeding cost vectors~$c_t = \hat \ell_t$ to the CSC oracle.

\subparagraph{IWR (importance-weighted regression).}
We consider a third method that directly reduces to the importance-weighted regression oracle~\eqref{eq:reg_oracle},
which we refer to as IWR,
and is suitable for algorithms which rely on off-policy learning.\footnote{Note that IWR is not directly applicable to methods that explicitly reduce to CSC oracles with well-chosen cost vectors, such as~\citet{agarwal2014taming,dudik2011efficient}.}
This approach finds a regressor
\begin{equation}
\label{eq:mtr}
\hat{f} := \arg\min_{f \in \mathcal{F}} \sum_{t=1}^T \frac{1}{p_t(a_t)}(f(x_t, a_t) - \ell_t(a_t))^2,
\end{equation}
and considers the policy~$\hat{\pi}(x) = \arg\min_a \hat{f}(x, a)$.
Such an estimator has been used, \eg, in the context of off-policy learning for recommendations~\citep{schnabel2016recommendations}
and is available in the Vowpal Wabbit library.
Note that if~$p_t$ has full support, then the objective is an unbiased
estimate of the full regression objective on all actions,
\begin{align*}
\sum_{t=1}^T \sum_{a=1}^K (f(x_t, a) - \ell_t(a))^2.
\end{align*}
In contrast, if the learner only explores a single action (so that $p_t(a_t) = 1$ for all~$t$),
the obtained regressor~$\hat f$ is the same as the loss estimator~$\hat \ell$ in~\eqref{eq:loss_estimator_def}.
In this csae,
if we consider a linear class of regressors of the form $f(x, a) = \theta_a^{\top} x$ with $x \in \R^d$,
then the \MTR{} reduction computes least-squares estimates~$\hat{\theta}_a$
from the data observed when action~$a$ was chosen.
When actions are selected according to the greedy policy $a_t = \arg\min_a \hat{\theta}_a^\top x_t$,
this setup corresponds to the greedy algorithm considered, \eg, by~\citet{bastani2017exploiting}.

Note that while CSC is typically intractable and requires approximations in order to work in practice,
importance-weighted regression does not suffer from these issues.
In addition, while the computational cost for an approximate CSC online update scales with the number of actions~$K$,
\MTR{} only requires an update for a single action,
making the approach more attractive computationally.
Another benefit of \MTR{} in an online setting is that it can leverage
importance weight aware online updates~\citep{karampatziakis2011online},
which makes it easier to handle large inverse propensity scores.

\subsection{Experimental Setup}
\label{sub:exp_setup}
Our experiments are conducted by simulating the contextual bandit setting using multiclass or cost-sensitive classification datasets,
and use the online learning system Vowpal Wabbit (VW).

\paragraph{Simulated contextual bandit setting.}
The experiments in this paper are based on leveraging supervised cost-sensitive classification datasets for simulating CB learning.
In particular, we treat a CSC example $(x_t, c_t) \in \mathcal{X} \times \R^K$
as a CB example, with $x_t$ given as the context to a CB algorithm,
and we only reveal the loss for the chosen action~$a_t$.
For a multiclass example with label~$y_t \in \{1, \ldots, K\}$, we set $c_t(a) := \1\{a \ne y_t\}$;
for multilabel examples with label set~$Y_t \subseteq \{1, \ldots, K\}$, we set $c_t(a) := \1\{a \notin Y_t\}$;
the cost-sensitive datasets we consider have $c_t \in [0,1]^K$.
We consider more general \emph{loss encodings} defined with an additive offset on the cost by:
\begin{equation}
\label{eq:loss_enc_def}
\ell_t^{c}(a) = c + c_t(a),
\end{equation}
for some~$c\in \R$.
Although some techniques attempt to remove a dependence on such encoding choices
through appropriately designed counterfactual loss estimators~\citep{dudik2011doubly,swaminathan2015self},
these may be imperfect in practice, and particularly in an online scenario.
The behavior observed for different choices of~$c$ allows us to get a sense of the
robustness of the algorithms to the scale of observed losses,
which might be unknown.
Separately, different values of~$c$ can lead to lower variance for loss estimation in different scenarios:
$c = 0$ might be preferred if $c_t(a)$ is often 0, while $c = -1$ is preferred
when $c_t(a)$ is often 1.
In order to have a meaningful comparison between different algorithms,
loss encodings, as well as supervised multiclass classification,
our evaluation metrics consider the original costs~$c_t$, and view the loss encodings as a hyperparameter or design choice.

\paragraph{Online learning in VW.}
Online learning is an important tool for having machine learning systems
that quickly and efficiently adapt to observed data~\citep{agarwal2016multiworld,he2014practical,mcmahan2013ad}.
We run our CB algorithms in an online fashion using Vowpal Wabbit:
instead of exact solutions of the optimization oracles from Section~\ref{sub:oracles},
we consider online variants of the CSC and regression oracles, which incrementally update
the policies or regressors with online gradient steps or variants thereof.
More details on their implementations are provided in Appendix~\ref{sec:alg_details}.
Note that in VW, online CSC itself reduces to multiple online regression problems in VW (one per action),
so that we are left with only online regression steps.
In order to provide more adaptivity to the wide range of datasets considered and to better handle importance weights in IWR, our online regression updates use adaptive~\citep{duchi2011adaptive},
normalized~\citep{ross2013normalized} and importance-weight-aware~\citep{karampatziakis2011online} gradient updates,
with a single tunable step-size parameter.

\paragraph{Parameterization.}
We consider linearly parameterized policies taking the form $\pi(x) = \arg\min_a \theta_a^\top x$,
or in the case of the IWR reduction, regressors $f(x, a) = \theta_a^\top x$. 
For the DR loss estimator, we use a similar linear parameterization $\hat{\ell}(x, a) = \phi_a^\top x$.
Some datasets in our evaluation have an action-dependent structure, with different feature
vectors~$x_a$ for different actions~$a$; in this case we use parameterizations of the form
$f(x, a) = \theta^\top x_a$, and $\hat{\ell}(x, a) = \phi^\top x_a$,
where the parameters~$\theta$ and~$\phi$ are shared across all actions.
We note that the algorithms we consider do not rely on these specific forms, and easily
extend to more complex, problem-dependent representations, in particular any linear parameterization~$f(x, a) = \theta^\top \Phi(x, a)$ with any choice of feature map~$\Phi(x, a)$, which includes the above-mentioned settings but can be more general.

\section{Algorithms}
\label{sec:algorithms}

In this section, we present the main algorithms we study in this paper,
along with simple modifications that achieve improved exploration efficiency.
All methods are based on the generic scheme in Algorithm~\ref{alg:cb}.
The function $\verb+explore+$ computes the exploration distribution~$p_t$ over actions,
and $\verb+learn+$ updates the algorithm's policies~$\pi$ or regressors~$f$.
These policies or regressors are seen as global variables, and their parameters are updated using the following routines, which implement in particular the oracles and loss estimators described in Section~\ref{sec:setup} (more details on their actual implementation are provided in Appendix~\ref{sec:alg_details}):
\begin{itemize}
  \item $\verb+csc_update+(\pi, (x, c))$: performs an online CSC update to the policy~$\pi$ using the CSC example~$(x, c) \in \mathcal X \times \R^K$.
  \item $\verb+reg_update+(f, (x, a, y, \omega))$: performs an online (importance-weighted) regression update to the regressor~$f$ on the example~$(x, a, y, \omega)$. When the importance weight~$\omega$ is not provided we assume a default value~$\omega = 1$.
  \item $\verb+estimator+(x_t, a_t, \ell_t(a_t), p_t)$: computes a vector of loss estimates~$\hat \ell_t \in \R^K$ using either IPS~\eqref{eq:ips} or DR~\eqref{eq:dr} from the interaction record~$(x_t, a_t, \ell_t(a_t), p_t)$. In the case of DR, this routine also performs an online regression update to a global loss estimator~$\hat \ell$ before computing the DR loss estimates, which takes the form~$\verb+reg_update+(\hat \ell, (x_t, a_t, \ell_t(a_t)))$.
  \item $\verb+opl_update+(\pi, (x_t, a_t, \ell_t(a_t), p_t))$: performs an online off-policy learning update using either IPS, DR or IWR using an interaction record~$(x_t, a_t, \ell_t(a_t), p_t)$. For IPS and DR this consists of a call to the corresponding~$\verb+estimator+$ routine to obtain the loss estimates~$\hat \ell_t$, followed by a call to~$\verb+csc_update+(\pi, (x_t, \hat \ell_t))$. For IWR, this is simply a call to~$\verb+reg_update+(f, (x_t, a_t, \ell_t(a_t), 1/p_t))$, where~$f$ is the underlying regressor that defines the policy as~$\pi(x) = \arg\min_a f(x, a)$.
\end{itemize}

Some of the algorithms we consider reduce directly to off-policy learning and hence only need to rely on the~$\verb+opl_update+$ routine, for a particular choice of loss estimator among IPS, DR or IWR.
Others rely on more ad hoc calls to underlying optimization oracles, in particular Cover and RegCB, and hence use the other routines as well.

Our C++ implementations of each algorithm are available in the Vowpal Wabbit online learning library.

\begin{algorithm}[tb]
\caption{Generic contextual bandit algorithm}
\label{alg:cb}
\begin{algorithmic}
\FOR{$t = 1, \ldots$}
	\STATE Observe context~$x_t$, compute $p_t := \verb+explore+(x_t)$;
	\STATE Choose action $a_t \sim p_t$, observe loss $\ell_t(a_t)$;
	\STATE $\verb+learn+(x_t, a_t, \ell_t(a_t), p_t)$;
\ENDFOR
\end{algorithmic}
\end{algorithm}

\subsection{$\epsilon$-greedy and greedy}
\label{sub:e_greedy}

\begin{algorithm}[tb]
\caption{$\epsilon$-greedy}
\label{alg:egreedy}
\textbf{Inputs}: exploration rate $\epsilon > 0$ (or $\epsilon = 0$ for Greedy), off-policy learning oracle \verb+opl_update+ (IPS/DR/IWR for $\epsilon$-Greedy, IWR for Greedy).

\textbf{Global state}: policy $\pi$.

$\verb+explore+(x_t)$:
\begin{algorithmic}
  \STATE {\bfseries return} $p_t(a) = \epsilon / K + (1 - \epsilon) \1\{\pi(x_t) = a\}$;
\end{algorithmic}
$\verb+learn+(x_t, a_t, \ell_t(a_t), p_t)$:
\begin{algorithmic}
  \STATE $\verb+opl_update+(\pi, (x_t, a_t, \ell_t(a_t), p_t(a_t)))$;
\end{algorithmic}
\end{algorithm}
We consider an importance-weighted variant of the epoch-greedy
approach of~\citet{langford2008epoch}, given in Algorithm~\ref{alg:egreedy}.
The method acts greedily with probability~$1 - \epsilon$,
and otherwise explores uniformly on all actions.
Learning is achieved by reduction
to off-policy optimization, through any of the three reductions presented in Section~\ref{sub:reductions}.

We also experimented with a variant we call active $\epsilon$-greedy,
that uses notions from disagreement-based active learning~\citep{hanneke2014theory,hsu2010algorithms}
in order to reduce uniform exploration to only actions that could plausibly be taken by the optimal policy.
While this variant often improves on the basic~$\epsilon$-greedy method, we found that it is often outperformed empirically
by other exploration algorithms, and thus defer its presentation to Appendix~\ref{sec:active_e_greedy_appx},
along with a theoretical analysis, for reference.

\paragraph{Greedy.}
When taking $\epsilon = 0$ in the $\epsilon$-greedy approach, with the IWR reduction,\footnote{We note that while IPS and DR could also be used here, the resulting algorithms are less natural since loss estimates have large bias when~$p_t$ is only supported on the greedy action, making them poor candidates for off-policy learning.
We also found these variants to be outperformed by IWR in our experiments.}
we are left with a fully greedy approach that always selects the action given by the current policy.
This gives us an online variant of the greedy algorithm of~\citet{bastani2017exploiting},
which regresses on observed losses and acts by selecting the action with minimum predicted loss.
Although this greedy strategy does not have an explicit mechanism for exploration in its choice of actions,
the inherent diversity in the distribution of contexts may provide sufficient exploration for
good performance and provable regret guarantees~\citep{bastani2017exploiting,kannan2018smoothed}.
In particular, under appropriate assumptions including a diversity assumption on the contexts,
one can show that all actions have a non-zero probability of being selected at each step,
providing a form of ``natural'' exploration from which one can establish regret guarantees.
Empirically, we find that Greedy can perform very well in practice on many datasets
(see Section~\ref{sec:experiments}).
If multiple actions get the same score according to the current regressor, we break ties randomly.

\subsection{Bag (Online Bootstrap Thompson Sampling)}
\label{sub:bag}

\begin{algorithm}[th]
\caption{Bag / Online BTS}
\label{alg:bag}
\textbf{Inputs}: number of policies~$N$, off-policy learning oracle \verb+opl_update+ (IPS/DR/IWR).

\textbf{Global state}: policies~$\pi^1, \ldots, \pi^N$.

$\verb+explore+(x_t)$:
\begin{algorithmic}
  \STATE {\bfseries return} $p_t(a) \propto |\{i : \pi^i(x_t) = a\}|$;\footnotemark
\end{algorithmic}
$\verb+learn+(x_t, a_t, \ell_t(a_t), p_t)$:
\begin{algorithmic}
  \FOR{$i = 1, \ldots, N$}
  	\STATE $\tau^i \sim Poisson(1)$; \hfill \COMMENT{or $\tau^1 := 1$ for bag-greedy}
    \FOR{$s = 1, \ldots, \tau^i$}
      \STATE $\verb+opl_update+(\pi^i, (x_t, a_t, \ell_t(a_t), p_t(a_t)))$;
    \ENDFOR
  \ENDFOR
\end{algorithmic}
\end{algorithm}

We now consider a variant of Thompson sampling which is usable in practice with optimization oracles.
Thompson sampling provides a generic approach to exploration problems, which maintains a belief
on the data generating model in the form of a posterior distribution given the observed data,
and explores by selecting actions according to a model sampled from this posterior~\citep[see, \eg,][]{agrawal2013thompson,chapelle2011empirical,russo2017tutorial,thompson1933likelihood}.
While the generality of this strategy makes it attractive, maintaining this posterior distribution
can be intractable for complex policy classes, and may require strong modeling assumptions.
In order to overcome such difficulties and to support the optimization oracles considered in this paper,
we rely on an approximation of Thompson sampling known as online Bootstrap Thompson
sampling~\citep[BTS,][]{eckles2014thompson,eckles2019bootstrap,osband2015bootstrapped}, or bagging~\citep{agarwal2014taming}.
This approach, shown in Algorithm~\ref{alg:bag}, maintains a collection of~$N$ policies $\pi^1, \ldots, \pi^N$
meant to approximate the posterior distribution over policies via the online
Bootstrap~\citep{agarwal2014taming,eckles2014thompson,osband2015bootstrapped,oza2001online,qin2013efficient},
and explores in a Thompson sampling
fashion, by computing the probability over
policies of each action in a bagging style (hence the name \emph{Bag}).

\footnotetext{When policies are parametrized using regressors as in our implementation, we let $\pi^i(x)$ be uniform over all actions tied for the lowest cost, and the final distribution is uniform across all actions tied for best according to one of the policies in the bag. The added randomization gives useful variance reduction in our experiments.}

Each policy is trained on a different online Bootstrap sample of the observed data, in the form of interaction records.
The online Bootstrap performs a random number $\tau$ of online
updates to each policy instead of one.
We use a Poisson distribution with parameter~1 for~$\tau$, which ensures that in expectation,
each policy is trained on~$t$ examples after~$t$ steps.
In contrast to~\citet{eckles2014thompson,osband2015bootstrapped}, which play the arm given by one of
the~$N$ policies chosen at random, we compute the full action
distribution~$p_t$ resulting from such a sampling, and leverage this
for loss estimation, allowing learning by reduction to off-policy optimization as in~\citet{agarwal2014taming}.
As in the~$\epsilon$-greedy algorithm, Bag directly relies on off-policy learning and thus all three reductions are admissible.

\paragraph{Bag-Greedy.}
We also consider a simple optimization that we call \emph{Bag-greedy},
for which the first policy~$\pi^1$ is trained on the true data sample (like Greedy), that is,
with~$\tau$ always equal to one, instead of a bootstrap sample with random choices of~$\tau$.
We found this approach to often improve on Bag, particularly when the number of policies~$N$ is~small.
We note that future work could consider other modifications of Bag which lead to an increasingly Greedy-like behavior controlled by an exploration parameter~$\epsilon$. For instance, putting a~$1- \epsilon$ weight in~$p_t$ on the action chosen either by~$\pi^1$, or by the majority of policies, and only~$\epsilon$ weight on the actions chosen by the other policies.

\subsection{Cover}
\label{sub:cover}

\begin{algorithm}[tb]
\caption{Cover}
\label{alg:cover}
\textbf{Inputs}: number of policies~$N$, exploration parameters $\epsilon_t = \min(1/K, 1/\sqrt{Kt})$ and $\psi > 0$, loss estimator \verb+estimator+ (IPS or DR).

\textbf{Global state}: policies $\pi^1, \ldots, \pi^N$.

$\verb+explore+(x_t)$:
\begin{algorithmic}
  \STATE $p_t(a) \propto |\{i : \pi^i(x_t) = a\}|$;
  \STATE {\bfseries return} $\epsilon_t + (1 - \epsilon_t) p_t$; \hfill \COMMENT{for cover}
  \STATE {\bfseries return} $p_t$; \hfill \COMMENT{for cover-nu}
\end{algorithmic}
$\verb+learn+(x_t, a_t, \ell_t(a_t), p_t)$:
\begin{algorithmic}
  \STATE $\hat{\ell}_t := \verb+estimator+(x_t, a_t, \ell_t(a_t), p_t(a_t))$;
  \STATE $\verb+csc_update+(\pi^1, (x_t, \hat \ell_t))$;
  \FOR{$i = 2, \ldots, N$}
    \STATE $q_i(a) \propto |\{j \leq i - 1 : \pi^j(x_t) = a\}|$;
    \STATE $\hat{c}(a) := \hat{\ell}_t(a) - \frac{\psi \epsilon_t}{\epsilon_t + (1 - \epsilon_t)q_i(a)}$;
  	\STATE $\verb+csc_update+(\pi^i, (x_t, \hat{c}))$;
  \ENDFOR
\end{algorithmic}
\end{algorithm}

This method, given in Algorithm~\ref{alg:cover}, is based on Online Cover,
an online approximation of the ``ILOVETOCONBANDITS'' algorithm of~\citet{agarwal2014taming}.
The approach maintains a collection of $N$ policies, $\pi^1, \ldots, \pi^N$,
meant to approximate a covering distribution over policies that are good for both exploration and exploitation.
The first policy $\pi^1$ is trained on observed data using the oracle as in previous algorithms,
while subsequent policies are trained using modified cost-sensitive examples which encourage diversity in
the predicted actions compared to the previous policies.
Since this method relies on CSC updates with well-chosen cost vectors, it does support general off-policy learning oracles, in particular IWR updates.
We note that we have tried simple variants of Cover based on the more computationally efficient IWR updates, but have not succeeded at obtaining one such variant that is competitive with DR in experiments.

Our implementation differs from the Online Cover algorithm of~\citet[Algorithm 5]{agarwal2014taming} in how
the diversity term in the definition of~$\hat{c}(a)$ is handled (the second term).
When creating cost-sensitive examples for a given policy~$\pi^i$,
this term rewards an action~$a$ that is not well-covered by previous policies (\ie, small~$q_i(a)$),
by subtracting from the cost a term that decreases with $q_i(a)$.
While Online Cover considers a fixed $\epsilon_t = \epsilon$,
we let~$\epsilon_t$ decay with~$t$, and introduce a parameter~$\psi$ to control the overall reward term,
which bears more similarity with the analyzed algorithm.
In particular, the magnitude of the reward is~$\psi$ whenever action~$a$ is not covered by previous policies
(\ie, $q_i(a) = 0$), but decays with $\psi \epsilon_t$ whenever $q_i(a) > 0$,
so that the level of induced diversity can decrease over time as we gain confidence
that good policies are covered.

\paragraph{Cover-NU.}
While Cover requires some uniform exploration across all actions, our experiments suggest that
this can make exploration highly inefficient, thus we introduce a variant, \emph{Cover-NU},
with \emph{n}o \emph{u}niform exploration outside the set of actions selected by covering policies.

\begin{algorithm}[tb]
\caption{RegCB}
\label{alg:regcb}
\textbf{Inputs}: exploration parameters~$C_0 > 0$ and~$\Delta_{t,C_0}$ as in~\eqref{eq:regcb_delta}.

\textbf{Global state}: regressor~$f$.

$\verb+explore+(x_t)$:
\begin{algorithmic}
  \STATE $l_t(a) := \verb+lcb+(f, x_t, a, \Delta_{t,C_0})$;
  \STATE $u_t(a) := \verb+ucb+(f, x_t, a, \Delta_{t,C_0})$;
  \STATE $p_t(a) \propto \1\{a \in \arg\min_{a'} l_t(a')\}$; \hfill \COMMENT{RegCB-opt variant}
  \STATE $p_t(a) \propto \1\{l_t(a) \leq \min_{a'} u_t(a')\}$; \hfill \COMMENT{RegCB-elim variant}
  \STATE {\bfseries return} $p_t$;
\end{algorithmic}
$\verb+learn+(x_t, a_t, \ell_t(a_t), p_t)$:
\begin{algorithmic}
  \STATE $\verb+reg_update+(f, (x_t, a_t, \ell_t(a_t)))$;
\end{algorithmic}
\end{algorithm}

\subsection{RegCB}
\label{sub:regcb}

We consider online approximations of the two algorithms introduced by~\citet{foster2018practical} based on regression oracles,
shown in Algorithm~\ref{alg:regcb}.
Both algorithms estimate confidence intervals of the loss for each action given the current context~$x_t$,
denoted~$[l_t(a), u_t(a)]$ in Algorithm~\ref{alg:regcb}.
The \emph{optimistic} variant then selects the action with smallest lower bound estimate, similar to LinUCB,
while the \emph{elimination} variant explores uniformly on actions that may plausibly be the best.

\paragraph{Confidence bounds with regression oracles.}
The confidence intervals are obtained by considering worst-case predictions over a confidence set of regressors with small excess squared loss.
Concretely, the RegCB algorithm studied by~\citet{foster2018practical}, which uses offline regression oracles rather than online oracles, defines the confidence bounds as follows:
\begin{equation}
\label{eq:lucb}
l_t(a) = \min_{f \in \mathcal F_t} f(x_t, a), \quad \text{and} \quad u_t(a) = \max_{f \in \mathcal F_t} f(x_t, a).
\end{equation}
Here,~$\mathcal F_t$ is a subset of regressors that is ``good'' for loss estimation,
in the sense that it achieves a small regression loss on observed data,
$\hat R_{t-1}(f) := \frac{1}{t-1} \sum_{s=1}^{t-1} (f(x_t, a_t) - \ell_t(a_t))^2$,
compared to the best regressor in the full regressor class~$\mathcal F$:
\[
\mathcal F_t := \{f \in \mathcal F : \hat R_{t-1}(f) - \min_{f \in \mathcal F} \hat R_{t-1}(f) \leq \Delta_t\},
\]
where~$\Delta_t$ is a quantity decreasing with~$t$ obtained from the theoretical analysis.

A key insight in~\citep{foster2018practical} is that the lower and upper confidence estimates given in~\eqref{eq:lucb} can be approximately computed using a small number of calls to a regression oracle.
If we consider the lower bound~$l_t(a)$ (the upper bound calculations are analogous), and if we assume that losses are in a known range~$[c_{\min}, c_{\max}]$, these oracle calls are based on the observed dataset along with one additional weighted example~$(x_t, a, y_l, \omega)$ with~$y_l = c_{\min} - 1$ (for the upper bound, $y_u = c_{\max} + 1$ is used instead), that is
\begin{equation}
\label{eq:fomega}
f^l_{a,\omega} := \arg\min_{f \in \mathcal F} ~(t-1) \hat R_{t-1}(f) + \omega (f(x_t, a) - y_l)^2,
\end{equation}
where~$\omega$ is an importance weight that is chosen as large as possible while ensuring that~$f^l_{a,\omega} \in \mathcal F_t$, that is,~$\hat R_{t-1}(f^l_{a,\omega}) - \min_f \hat R_{t-1}(f) \leq \Delta_t$.
Then it can be shown that~$l_t(a) \approx f^l_{a,\omega}(x_t, a)$, with an accuracy that improves exponentially with the number of oracle calls (corresponding to the number of binary search iterations over the choice of~$\omega$).

\paragraph{Online approximation of confidence bounds.}
In order to compute confidence bounds in our online setup, we approximate the above procedure using an importance-weight sensitivity analysis of online regression, similar to the approach described in~\citep[Section 7.1]{krishnamurthy2017active} in the context of active learning.
The key idea is that one may approximate~$f^l_{a,\omega}$ in~\eqref{eq:fomega} by performing an importance-weighted online regression update with example~$(x_t, a, y_l, \omega)$ to the regressor trained on the previous observed examples, which is denoted~$f$ in Algorithm~\ref{alg:regcb}.
Since such an update to~$f$'s parameter using online gradient descent is typically proportional to~$\omega$, we may then consider the approximation
\begin{equation}
\label{eq:fomega_approx}
f^l_{a,\omega}(x_t, a) \approx f(x_t, a) + s^l_a \omega,
\end{equation}
where~$s^l_a$ is the \emph{sensitivity} of~$f(x_t, a)$ to the update~$(x_t, a, y_l, \omega)$, given by the derivative of the output of the updated regressor w.r.t.~$\omega$ (which is negative here assuming~$f(x_t, a) \geq c_{\min}$).
It remains to find a maximal value of~$\omega$, denoted~$\omega^*$, such that the resulting updated regressor is still ``close'' to~$f$ in the sense of excess squared loss.
In order to find a reasonable method for finding such~$\omega^*$, it is helpful to reason again in terms of offline oracles.
Denoting~$f_t = \arg\min \hat R_{t-1}(f)$, we would like to consider
\[
\omega^* = \max\{\omega \text{ s.t. } \hat R_{t-1}(f^l_{a,\omega}) - \hat R_{t-1}(f_t) \leq \Delta_t\},
\]
where~$f^l_{a,\omega}$ is given in~\eqref{eq:fomega}. \citet{krishnamurthy2017active} show the following upper bound:
\[
(t-1) (\hat R_{t-1}(f^l_{a,\omega}) - \hat R_{t-1}(f_t)) \leq \omega (f_t(x_t, a) - y_l)^2 - \omega (f^l_{a,\omega}(x_t, a) - y_l)^2.
\]
We may then choose~$\omega^*$ by maximizing this upper bound, using the approximation~\eqref{eq:fomega_approx} instead of~$f^l_{a,\omega}(x_t, a)$ and the online estimate~$f$ from Algorithm~\ref{alg:regcb} instead of~$f_t$.
This leads to the following implementation of $\verb+lcb+$ in Algorithm~\ref{alg:regcb}:
\begin{align*}
l_t(a) &:= f(x_t, a) + s^l_a \omega^* \\
\text{with~~} \omega^* &:= \max\{\omega \text{ s.t. } \omega (f(x_t, a) - y_l)^2 - \omega (f(x_t, a) + s^l_a \omega - y_l)^2 \leq (t-1) \Delta_t\},
\end{align*}
where the optimization over~$\omega$ can be done with a simiple binary search procedure in the range $[0, (f(x_t, a) - y_l) / (-s^l_a)]$, since one can verify that the objective is increasing in this range.
In practice, we replace the theoretical value $(t-1) \Delta_t$ by~$\Delta_{t,C_0}$ given by
\begin{equation}
\label{eq:regcb_delta}
\Delta_{t,C_0} = C_0 \log (Kt),
\end{equation}
where~$C_0$ is a parameter controlling the width of the confidence bounds.
Note that unlike other methods, this algorithm requires knowledge of the loss range~$[c_{\min}, c_{\max}]$ (which defines the artificial labels~$y^l$ and~$y^u$ used in $\verb+lcb+/\verb+ucb+$).

\section{Evaluation}
\label{sec:experiments}

In this section, we present our evaluation of the contextual bandit algorithms described in Section~\ref{sec:algorithms}.
The evaluation code is available at \url{https://github.com/albietz/cb_bakeoff}.
All methods presented in this section are available in Vowpal Wabbit.\footnote{For reproducibility purposes, the precise version of VW used to run these experiments is available at \url{https://github.com/albietz/vowpal_wabbit/tree/bakeoff}.}

\paragraph{Evaluation setup.}
Our evaluation consists in simulating a CB setting from
cost-sensitive classification datasets, as described in Section~\ref{sub:exp_setup}.
We consider a collection of 516 multiclass classification datasets from the \url{openml.org} platform,
including among others, medical, gene expression, text, sensory or synthetic data,
as well as 5 multilabel datasets\footnote{\url{https://www.csie.ntu.edu.tw/~cjlin/libsvmtools/datasets/multilabel.html}}
and 3 cost-sensitive datasets,
namely a cost-sensitive version of the RCV1 multilabel dataset used in~\citep{krishnamurthy2017active},
where the cost of a news topic is equal to the tree distance to a correct topic,
as well as the two learning to rank datasets used in~\citep{foster2018practical}.
More details on these datasets are given in Appendix~\ref{sec:datasets}.
Because of the online setup, we consider one or more fixed, shuffled orderings of each dataset. The datasets widely vary in noise levels, and number of actions, features, examples etc., allowing us to model varying difficulties in CB problems.

We evaluate the algorithms described in Section~\ref{sec:algorithms}.
We ran each method on every dataset with different choices of algorithm-specific hyperparameters,
learning rates, reductions, and loss encodings.
Details are given in Appendix~\ref{sub:hyperparams_appx}.
Unless otherwise specified, we consider \emph{fixed choices} which are chosen to optimize performance on a subset of multiclass datasets with a voting mechanism
and are highlighted in Table~\ref{table:algos} of Appendix~\ref{sec:evaluation_appx},
except for the learning rate, which is always optimized.

The performance of method~$\mathcal A$ on a dataset of size $n$ is measured by the
\textbf{progressive validation loss}~\citep{blum1999beating}:
\[
PV_{\mathcal A} = \frac{1}{n} \sum_{t=1}^{n} c_t(a_t),
\]
where~$a_t$ is the action chosen by the algorithm on the $t$-th example,
and~$c_t$ the true cost vector.
This metric allows us to capture the explore-exploit trade-off,
while providing a measure of generalization that is independent of the choice of loss encodings,
and comparable with online supervised learning.
We also consider a \emph{normalized loss} variant given by $\frac{PV_{\mathcal A} - PV_{\text{OAA}}}{PV_{\text{OAA}}}$,
where $\text{OAA}$ denotes an online (supervised) cost-sensitive one against all classifier.
This helps highlight the difficulty of exploration for some datasets in our plots.

In order to compare two methods on a given dataset with binary costs (multiclass or multilabel),
we consider a notion of \textbf{statistically significant win or loss}.
We use the following (heuristic) definition of significance based on an approximate Z-test:
if~$p_a$ and~$p_b$ denote the PV loss of~$a$ and~$b$ on a given dataset of size~$n$,
then $a$ wins over~$b$ if
\[1 - \Phi \left(\frac{p_a - p_b}{\sqrt{\frac{p_a(1 - p_a)}{n} + \frac{p_b(1 - p_b)}{n}}} \right) < 0.05, \]
where $\Phi$ is the Gauss error function.
We also define the \emph{significant win-loss difference} of one algorithm against another
to be the difference between the number of significant wins and significant losses.
We have found these metrics to provide more insight into the behavior of different methods,
compared to strategies based on aggregation of loss measures across all datasets.
Indeed, we often found the relative performance of two methods to vary significantly across datasets,
making aggregate metrics less informative.

Results in this section focus on Greedy (G), RegCB-optimistic (RO), Cover-NU (C-nu), Bag/BTS-greedy (B-g)
and $\epsilon$-greedy ($\epsilon$G), deferring other variants to Appendix~\ref{sec:evaluation_appx},
as their performance is typically comparable to or dominated by these methods.
We combine results on multiclass and multilabel datasets, but show them separately in Appendix~\ref{sec:evaluation_appx}.

\begin{table}[tb]
\caption{Entry (row, column) shows the \emph{statistically significant} win-loss difference
of row against column. Encoding fixed to -1/0 (top) or 0/1 (bottom).
(left) held-out datasets only (322 in total); fixed hyperparameters, only the learning rate is optimized;
(right) all datasets (521 in total); all choices of hyperparameters are optimized on each dataset
(different methods have different hyperparameter settings from 1 to 18, see Table~\ref{table:algos} in Appendix~\ref{sec:evaluation_appx}).}
\label{table:win_loss_diff}

\small
\center
\begin{tabular}{ | l | c | c | c | c | c |  }
\hline
$\downarrow$ vs $\rightarrow$ & G & RO & C-nu & B-g & $\epsilon$G \\ \hline
G & - & -9 & 11 & 48 & 53 \\ \hline
\textbf{RO} & \textbf{9} & - & \textbf{28} & \textbf{50} & \textbf{69} \\ \hline
C-nu & -11 & -28 & - & 21 & 54 \\ \hline
B-g & -48 & -50 & -21 & - & 18 \\ \hline
$\epsilon$G & -53 & -69 & -54 & -18 & - \\ \hline
\end{tabular}
%
%
~~\begin{tabular}{ | l | c | c | c | c | c |  }
\hline
$\downarrow$ vs $\rightarrow$ & G & RO & C-nu & B-g & $\epsilon$G \\ \hline
G & - & -32 & -88 & -21 & 59 \\ \hline
RO & 32 & - & -33 & 12 & 87 \\ \hline
\textbf{C-nu} & \textbf{88} & \textbf{33} & - & \textbf{60} & \textbf{140} \\ \hline
B-g & 21 & -12 & -60 & - & 86 \\ \hline
$\epsilon$G & -59 & -87 & -140 & -86 & - \\ \hline
\end{tabular}

\vspace{0.1cm}
-1/0 encoding
\vspace{0.3cm}

\begin{tabular}{ | l | c | c | c | c | c |  }
\hline
$\downarrow$ vs $\rightarrow$ & G & RO & C-nu & B-g & $\epsilon$G \\ \hline
G & - & -62 & -17 & 38 & 51 \\ \hline
\textbf{RO} & \textbf{62} & - & \textbf{44} & \textbf{100} & \textbf{117} \\ \hline
C-nu & 17 & -44 & - & 48 & 74 \\ \hline
B-g & -38 & -100 & -48 & - & 15 \\ \hline
$\epsilon$G & -51 & -117 & -74 & -15 & - \\ \hline
\end{tabular}
~~\begin{tabular}{ | l | c | c | c | c | c |  }
\hline
$\downarrow$ vs $\rightarrow$ & G & RO & C-nu & B-g & $\epsilon$G \\ \hline
G & - & -121 & -145 & -69 & 6 \\ \hline
RO & 121 & - & -10 & 68 & 122 \\ \hline
\textbf{C-nu} & \textbf{145} & \textbf{10} & - & \textbf{69} & \textbf{141} \\ \hline
B-g & 69 & -68 & -69 & - & 78 \\ \hline
$\epsilon$G & -6 & -122 & -141 & -78 & - \\ \hline
\end{tabular}

\vspace{0.1cm}
0/1 encoding

\end{table}

\begin{table}[tb]
\caption{Progressive validation loss for cost-sensitive datasets with real-valued costs in~$[0,1]$.
Hyperparameters are fixed as in Table~\ref{table:hyperparams2}.
We show mean and standard error based on~10 different random reshufflings.
For RCV1, costs are based on tree distance to correct topics.
For MSLR and Yahoo, costs encode 5 regularly-spaced discrete relevance scores (0: perfectly relevant, 1: irrelevant), and we include results for a loss encoding offset~$c = -1$ in Eq.~\eqref{eq:loss_enc_def}.
}
\label{table:cost_sensitive}
\small
\center
\begin{tabular}{ | c | c | c | c | c |  }
\hline
G & RO & C-nu & B-g & $\epsilon$G \\ \hline
\textbf{0.215} $\pm$ 0.010 & 0.225 $\pm$ 0.008 & \textbf{0.215} $\pm$ 0.006 & 0.251 $\pm$ 0.005 & 0.230 $\pm$ 0.009 \\ \hline
\end{tabular}

(a) RCV1
\vspace{0.1cm}

\begin{tabular}{ | c | c | c | c | c |  }
\hline
G & RO & C-nu & B-g & $\epsilon$G \\ \hline
0.798 $\pm$ 0.0023 & \textbf{0.794} $\pm$ 0.0007 & 0.798 $\pm$ 0.0012 & 0.799 $\pm$ 0.0013 & 0.807 $\pm$ 0.0020 \\ \hline
\end{tabular}

(b) MSLR
\vspace{0.1cm}

\begin{tabular}{ | c | c | c | c | c |  }
\hline
G & RO & C-nu & B-g & $\epsilon$G \\ \hline
0.791 $\pm$ 0.0012 & \textbf{0.790} $\pm$ 0.0009 & 0.792 $\pm$ 0.0007 & 0.791 $\pm$ 0.0008 & 0.806 $\pm$ 0.0018 \\ \hline
\end{tabular}

(c) MSLR, $c = -1$
\vspace{0.1cm}

\begin{tabular}{ | c | c | c | c | c |  }
\hline
G & RO & C-nu & B-g & $\epsilon$G \\ \hline
0.593 $\pm$ 0.0004 & \textbf{0.592} $\pm$ 0.0005 & 0.594 $\pm$ 0.0004 & 0.596 $\pm$ 0.0006 & 0.598 $\pm$ 0.0005 \\ \hline
\end{tabular}

(d) Yahoo
\vspace{0.1cm}

\begin{tabular}{ | c | c | c | c | c |  }
\hline
G & RO & C-nu & B-g & $\epsilon$G \\ \hline
0.589 $\pm$ 0.0005 & \textbf{0.588} $\pm$ 0.0004 & 0.589 $\pm$ 0.0004 & 0.590 $\pm$ 0.0005 & 0.594 $\pm$ 0.0006 \\ \hline
\end{tabular}

(e) Yahoo, $c = -1$

\end{table}

\paragraph{Efficient exploration methods.}
Our experiments suggest that the best performing method is the RegCB approach~\citep{foster2018practical},
as shown in Table~\ref{table:win_loss_diff} (left), where the significant wins of RO
against all other methods exceed significant losses,
which yields the best performance overall.
This is particularly prominent with 0/1 encodings.
With -1/0 encodings, which are generally preferred on our corpus as discussed below,
the simple greedy approach comes a close second, outperforming
other methods on a large number of datasets,
despite the lack of an explicit exploration mechanism.
A possible reason for this success is the diversity that is inherently present
in the distribution of contexts across actions, which has been shown to yield
no-regret guarantees under various assumptions~\citep{bastani2017exploiting,kannan2018smoothed}.
The noise induced by the dynamics of online learning and random tie-breaking
may also be a source of more exploration.
RO and Greedy also show strong performance on the 8 UCI datasets and the learning-to-rank datasets from~\citet{foster2018practical}, as shown in Tables~\ref{table:cost_sensitive} and~\ref{table:win_loss_uci}.
Nevertheless, both Greedy and RegCB have known failure modes which the Cover approach is robust to by design. While the basic approach with
uniform exploration is too conservative, we found our Cover-NU variant to be quite competitive overall.
The randomization in its choice of actions yields exploration logs which may be additionally used for offline evaluation, in contrast to Greedy and RegCB-opt,
which choose actions deterministically.
A more granular comparison of these methods is given
in Figure~\ref{fig:comp},
which highlight the failure of Greedy and RegCB against Cover-NU on some datasets which may be more difficult perhaps due to a failure of modeling assumptions.
Bag also outperforms other methods on some datasets, however it is outperformed on most datasets,
possibly because of the additional variance induced by the bootstrap sampling.
Table~\ref{table:win_loss_diff} (right) optimizes over hyperparameters for each dataset,
which captures the best potential of each method.
Cover-NU does the best here, but also has the most hyperparameters,
indicating that a more adaptive variant could be desirable.
RegCB stays competitive, while Greedy pales possibly due to fewer hyperparameters.

\begin{figure}[tb]
	\centering
	\includegraphics[width=0.30\columnwidth]{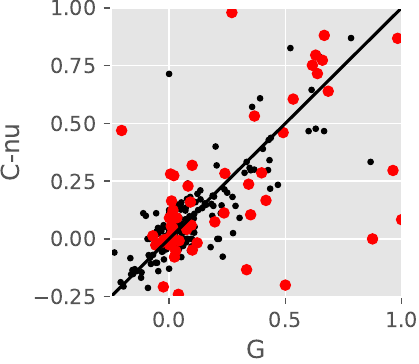}
	\includegraphics[width=0.30\columnwidth]{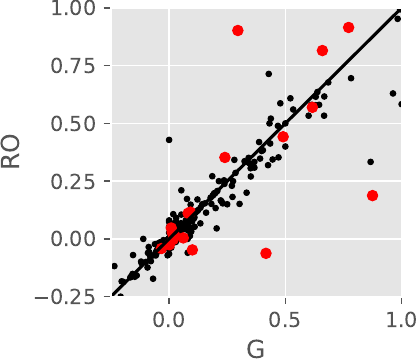}
	\includegraphics[width=0.30\columnwidth]{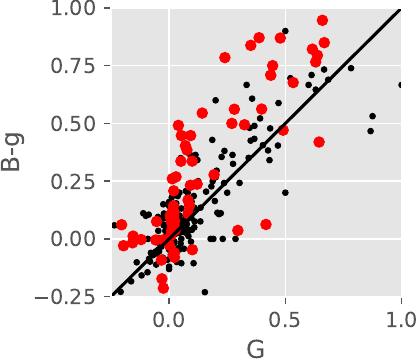}
	\includegraphics[width=0.30\columnwidth]{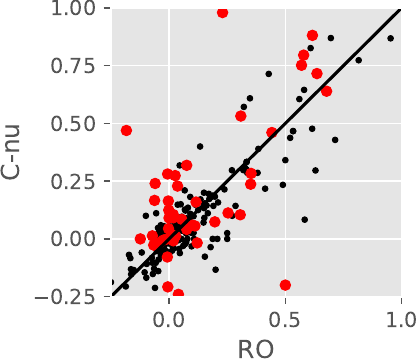}
	\includegraphics[width=0.30\columnwidth]{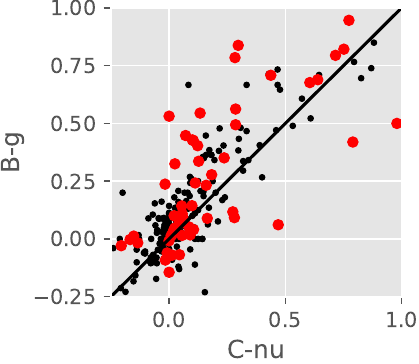}
	\includegraphics[width=0.30\columnwidth]{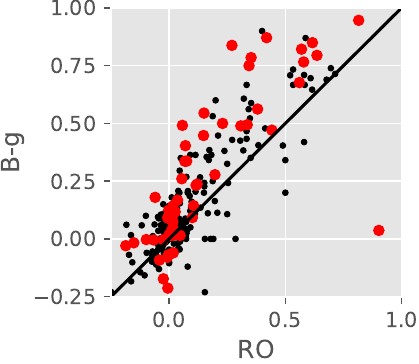}

	\caption{Pairwise comparisons among four successful methods:
	Greedy, Cover-nu, Bag-greedy, and RegCB-opt.
	Hyperparameters fixed as in Table~\ref{table:algos}, with encoding -1/0.
	All held-out multiclass and multilabel datasets are shown, in contrast to Figure~\ref{fig:greed_is_good},
	which only shows held-out datasets with 5 or more actions.
	The plots consider normalized loss, with red points indicating significant wins.
	}
	\label{fig:comp}
\end{figure}

\paragraph{Variability with dataset characteristics.}
Table~\ref{table:breakdown} shows win-loss statistics for
subsets of the datasets with constraints on different characteristics, such as
number of actions, dimensionality, size, and performance in the supervised setting.
The values for these splits were chosen in order to have a reasonably balanced number
of datasets in each table.

We find that RegCB-opt is the preferred method in most situations,
while Greedy and Cover-NU can also provide good performance in different settings.
When only considering larger datasets, RegCB-opt dominates all other methods, with Greedy
a close second, while Cover-NU seems to explore less efficiently.
This could be related to the better adaptivity properties of RegCB to favorable noise conditions,
which can achieve improved (even logarithmic) regret~\citep{foster2018practical},
in contrast to Cover-NU, for which a slower rate (of $O(\sqrt{n})$)
may be unavoidable since it is baked into the algorithm of~\citet{agarwal2014taming} by design,
and is reflected in the diversity terms of the costs~$\hat{c}$ in our online variant given in Algorithm~\ref{alg:cover}.
In contrast, when~$n$ is small, RegCB-opt and Greedy may struggle to find good policies
(in fact, their analysis typically requires a number of ``warm-start'' iterations with uniform exploration),
while Cover-NU seems to explore more efficiently from the beginning,
and behaves well with large action spaces or high-dimensional features.
Finally, Table~\ref{table:breakdown}(d,e) shows that Greedy can be the best choice when
the dataset is ``easy'', in the sense that a supervised learning method achieves small loss.
Achieving good performance on such easy datasets is related to the open problem of~\citet{agarwal2017open},
and variants of methods designed to be agnostic to the data distribution---such as Cover(-NU)
and $\epsilon$-Greedy~\citep{agarwal2014taming,langford2008epoch}---seem to be the weakest on these~datasets.

\begin{table}[tb]
\caption{Impact of reductions for Bag (left) and $\epsilon$-greedy (right),
with hyperparameters optimized and encoding fixed to -1/0.
Each (row, column) entry shows the \emph{statistically significant} win-loss difference
of row against column.
IWR outperforms the other reductions for both methods, which are the only two methods that directly reduce to off-policy learning, and thus where such a comparison applies.}
\label{table:reductions}

\small
\center
\begin{tabular}{ | l | c | c | c |  }
\hline
$\downarrow$ vs $\rightarrow$ & ips & dr & iwr \\ \hline
ips & - & -42 & -55 \\ \hline
dr & 42 & - & -25 \\ \hline
\textbf{iwr} & \textbf{55} & \textbf{25} & - \\ \hline
\end{tabular}
~~~\begin{tabular}{ | l | c | c | c |  }
\hline
$\downarrow$ vs $\rightarrow$ & ips & dr & iwr \\ \hline
ips & - & 61 & -128 \\ \hline
dr & -61 & - & -150 \\ \hline
\textbf{iwr} & \textbf{128} & \textbf{150} & - \\ \hline
\end{tabular}
\end{table}

\begin{table}[tb]
\caption{Impact of encoding on different algorithms, with hyperparameters optimized.
Each entry indicates the number of \emph{statistically significant} wins/losses of -1/0 against 0/1.
-1/0 is the better overall choice of encoding, but 0/1 can be preferable on larger datasets
(the bottom row considers the 64 datasets in our corpus with more than 10,000 examples).
}
\label{table:encoding}
\small
\center
\begin{tabular}{ | c | c | c | c | c | c |  }
\hline
datasets & G & RO & C-nu & B-g & $\epsilon$G \\ \hline
all & 132 / 42 & 58 / 46 & 71 / 46 & 73 / 27 & 94 / 27 \\ \hline
$\geq$ 10,000 & 19 / 12 & 10 / 18 & 14 / 20 & 15 / 11 & 14 / 5 \\ \hline
\end{tabular}

\end{table}

\paragraph{Reductions.}
Among the reduction mechanisms introduced in Section~\ref{sub:reductions},
IWR has desirable properties such as
tractability (the other reductions rely on a CSC objective, which requires approximations due to non-convexity),
and a computational cost that is independent
of the total number of actions, only requiring updates for the chosen action.
In addition to Greedy, which can be seen as using a form of IWR,
we found IWR to work very well for
Bag and $\epsilon$-Greedy, as shown in Table~\ref{table:reductions}
(see also Table~\ref{table:algos} in Appendix~\ref{sec:evaluation_appx}, which shows
that IWR is also preferred when considering fixed hyperparameters for these methods).
This may be attributed to the difficulty of the CSC problem compared to regression,
as well as importance weight aware online updates, which can be helpful for small~$\epsilon$.
Together with its computational benefits, our results suggest
that IWR is often a compelling alternative to CSC reductions based on IPS or DR.
In particular, when the number of actions is prohibitively large for using Cover-NU or RegCB,
Bag with IWR may be a good default choice of exploration algorithm.
While Cover-NU does not directly support the IWR reduction,
making them work together well would be a promising future direction.

\paragraph{Encodings.}
Table~\ref{table:encoding} indicates that the -1/0 encoding is preferred to 0/1
on many of the datasets, and for all methods.
We now give one possible explanation.
As discussed in Section~\ref{sub:exp_setup}, the -1/0 encoding yields low variance
loss estimates when the cost is often close to 1.
For datasets with binary costs, since the learner may often be wrong in early iterations,
a cost of 1 is a good initial bias for learning. With enough data, however, the learner should reach better accuracies and observe losses closer to 0,
in which case the 0/1 encoding should lead to lower variance estimates,
yielding better performance as observed in Table~\ref{table:encoding}.
We tried shifting the loss range in the RCV1 dataset with real-valued costs from~$[0,1]$ to~$[-1,0]$,
but saw no improvements compared to the results in Table~\ref{table:cost_sensitive}.
Indeed, a cost of 1 may not be a good initial guess in this case, in contrast to the binary cost setting.
On the MSLR and Yahoo learning-to-rank datasets, we do see some improvement from shifting costs
to the range~$[-1,0]$, perhaps because in this case the costs are discrete values, and the cost of
1 corresponds to the document label ``irrelevant'', which appears frequently in these datasets.

\begin{figure}[tb]
	\centering
	\begin{subfigure}[c]{.4\textwidth}
	\includegraphics[width=\textwidth]{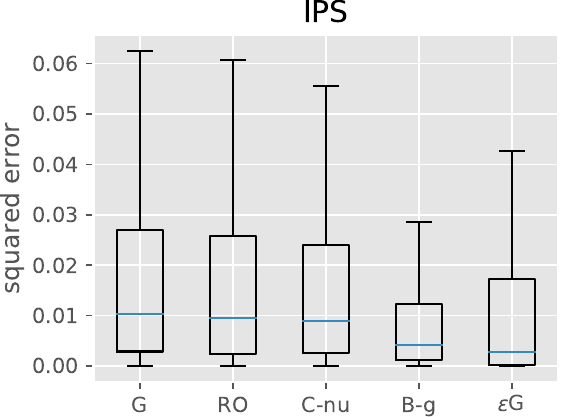}
	\caption{all multiclass datasets}
	\end{subfigure}
	\begin{subfigure}[c]{.4\textwidth}
	\includegraphics[width=\textwidth]{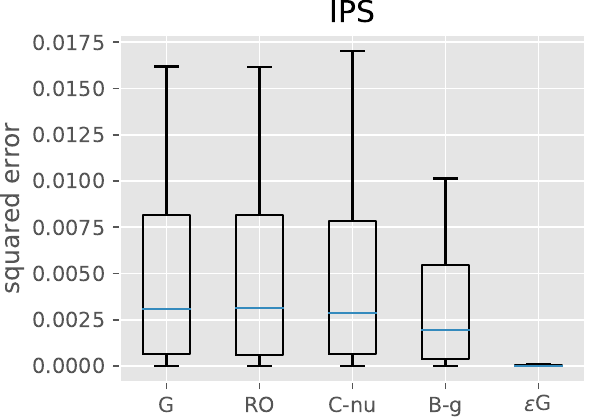}
	\caption{$n \geq 10\,000$ only}
	\label{fig:cfe_large_n}
	\end{subfigure}
	\caption{Errors of IPS counterfactual estimates for the uniform random policy using exploration
	logs collected by various algorithms on multiclass datasets.
	The boxes show quartiles (with the median shown as a blue line) of the distribution of squared errors across all multiclass datasets or only those with at least 10\,000 examples.
	The logs are obtained by running each algorithm with -1/0 encodings, fixed hyperparameters from Table~\ref{table:algos}, and the best learning rate on each dataset according to progressive validation loss.}
	\label{fig:cfe}
\end{figure}

\paragraph{Counterfactual evaluation.}
After running an exploration algorithm, a desirable goal for practitioners is to evaluate
new policies offline,
given access to interaction logs from exploration data.
While various exploration algorithms rely on such counterfactual evaluation through a reduction,
the need for efficient exploration might restrict the ability to evaluate arbitrary policies offline,
since the algorithm may only need to estimate ``good'' policies in the sense of small regret with respect to the optimal policy.
To illustrate this, in Figure~\ref{fig:cfe}, we consider the evaluation of the uniform random \emph{stochastic} policy,
given by $\pi^{\text{unif}}(a|x) = 1/K$ for all actions $a = 1, \ldots, K$,
using a simple inverse propensity scoring (IPS) approach.
While such a task is somewhat extreme and of limited practical interest in itself,
we use it mainly as a proxy for the ability to evaluate any \emph{arbitrary} policy (in particular, it should be possible to evaluate any policy if we can evaluate the uniform random policy).
On a multiclass dataset, the expected loss of such a policy is simply $1 - 1/K$,
while its IPS estimate is given by
\[
\hat L^{IPS} = 1 - \frac{1}{n} \sum_{t=1}^n \frac{1 - \ell_t(a_t)}{K p_t(a_t)},
\]
where the loss is given by~$\ell_t(a_t) = \1\{a_t \ne y_t\}$ when the correct multiclass label is~$y_t$.
Note that this quantity is well-defined since the denominator is always non-zero when~$a_t$ is sampled from~$p_t$, but the estimator is biased when data is collected by a method without some amount of uniform exploration (\ie, when~$p_t$ does not have full support).  This bias is particularly evident in Figure~\ref{fig:cfe_large_n} where epsilon-greedy shows very good counterfactual evaluation performance while the other algorithms induce biased counterfactual evaluation due to lack of full support.
The plots in Figure~\ref{fig:cfe} show the distribution of squared errors $(\hat L^{IPS} - (1 - 1/K))^2$ across multiclass datasets.
We consider IPS on the \emph{rewards}~$1 - \ell_t(a_t)$ here as it is more adapted to the -1/0 encodings used to collect exploration logs, but we also show IPS on losses in Figure~\ref{fig:cfe_appx} of Appendix~\ref{sub:results_appx}.
Figure~\ref{fig:cfe} shows that the more efficient exploration methods (Greedy, RegCB-optimistic, and Cover-NU) give poor estimates for this policy, probably because their
exploration logs provide biased estimates and are quite focused on few actions that may be taken by good policies,
while the uniform exploration in $\epsilon$-Greedy (and Cover-U, see Figure~\ref{fig:cfe_appx}) yields better estimates, particularly on larger datasets.
The elimination version of RegCB provides slightly better logs compared to the optimistic version (see Figure~\ref{fig:cfe_appx}), and Bag may also be preferred in this context.
Overall, these results show that there may be a trade-off between efficient exploration and the
ability to perform good counterfactual evaluation, and that uniform exploration may be needed
if one would like to perform accurate offline experiments for a broad range of questions.

\begin{table}[tb]
	\caption{\emph{Statistically significant} win-loss difference with hyperparameters fixed as in Table~\ref{table:hyperparams2},
	encodings fixed to -1/0, on all held-out datasets or subsets with different characteristics:
	(a) number of actions~$K$;
	(b) number of features~$d$;
	(c) number of examples~$n$;
	(d) PV loss of the one-against-all (OAA) method.
	The corresponding table for all held-out datasets is shown in Table~\ref{table:win_loss_diff}(top left).}
	\label{table:breakdown}
\small
\center

\begin{tabular}{ | l | c | c | c | c | c |  }
\hline
$\downarrow$ vs $\rightarrow$ & G & RO & C-nu & B-g & $\epsilon$G \\ \hline
G & - & -7 & -4 & 20 & 18 \\ \hline
\textbf{RO} & \textbf{7} & - & \textbf{8} & \textbf{25} & \textbf{27} \\ \hline
C-nu & 4 & -8 & - & 23 & 24 \\ \hline
B-g & -20 & -25 & -23 & - & -4 \\ \hline
$\epsilon$G & -18 & -27 & -24 & 4 & - \\ \hline
\end{tabular}
~\begin{tabular}{ | l | c | c | c | c | c |  }
\hline
$\downarrow$ vs $\rightarrow$ & G & RO & C-nu & B-g & $\epsilon$G \\ \hline
G & - & -1 & -2 & 13 & 7 \\ \hline
\textbf{RO} & \textbf{1} & - & \textbf{3} & \textbf{15} & \textbf{11} \\ \hline
C-nu & 2 & -3 & - & 20 & 11 \\ \hline
B-g & -13 & -15 & -20 & - & -7 \\ \hline
$\epsilon$G & -7 & -11 & -11 & 7 & - \\ \hline
\end{tabular}

\vspace{0.1cm}
(a) $K \geq 3$ (left, 83 datasets), $K \geq 10$ (right, 31 datasets)
\vspace{0.4cm}

\begin{tabular}{ | l | c | c | c | c | c |  }
\hline
$\downarrow$ vs $\rightarrow$ & G & RO & C-nu & B-g & $\epsilon$G \\ \hline
G & - & -5 & -1 & 20 & 10 \\ \hline
\textbf{RO} & \textbf{5} & - & \textbf{8} & \textbf{26} & \textbf{21} \\ \hline
C-nu & 1 & -8 & - & 20 & 17 \\ \hline
B-g & -20 & -26 & -20 & - & -7 \\ \hline
$\epsilon$G & -10 & -21 & -17 & 7 & - \\ \hline
\end{tabular}
~\begin{tabular}{ | l | c | c | c | c | c |  }
\hline
$\downarrow$ vs $\rightarrow$ & G & RO & C-nu & B-g & $\epsilon$G \\ \hline
G & - & -5 & -6 & 7 & 0 \\ \hline
\textbf{RO} & \textbf{5} & - & \textbf{1} & \textbf{8} & \textbf{8} \\ \hline
C-nu & 6 & -1 & - & 15 & 7 \\ \hline
B-g & -7 & -8 & -15 & - & -6 \\ \hline
$\epsilon$G & 0 & -8 & -7 & 6 & - \\ \hline
\end{tabular}

\vspace{0.1cm}
(b) $d \geq 100$ (left, 76 datasets), $d \geq 10\,000$ (right, 41 datasets)
\vspace{0.4cm}

\begin{tabular}{ | l | c | c | c | c | c |  }
\hline
$\downarrow$ vs $\rightarrow$ & G & RO & C-nu & B-g & $\epsilon$G \\ \hline
G & - & -5 & 24 & 37 & 52 \\ \hline
\textbf{RO} & \textbf{5} & - & \textbf{34} & \textbf{41} & \textbf{60} \\ \hline
C-nu & -24 & -34 & - & 8 & 34 \\ \hline
B-g & -37 & -41 & -8 & - & 18 \\ \hline
$\epsilon$G & -52 & -60 & -34 & -18 & - \\ \hline
\end{tabular}
~\begin{tabular}{ | l | c | c | c | c | c |  }
\hline
$\downarrow$ vs $\rightarrow$ & G & RO & C-nu & B-g & $\epsilon$G \\ \hline
G & - & -3 & 16 & 10 & 34 \\ \hline
\textbf{RO} & \textbf{3} & - & \textbf{17} & \textbf{14} & \textbf{36} \\ \hline
C-nu & -16 & -17 & - & 2 & 30 \\ \hline
B-g & -10 & -14 & -2 & - & 25 \\ \hline
$\epsilon$G & -34 & -36 & -30 & -25 & - \\ \hline
\end{tabular}

\vspace{0.1cm}
(c) $n \geq 1\,000$ (left, 113 datasets), $n \geq 10\,000$ (right, 43 datasets)
\vspace{0.4cm}

\begin{tabular}{ | l | c | c | c | c | c |  }
\hline
$\downarrow$ vs $\rightarrow$ & G & RO & C-nu & B-g & $\epsilon$G \\ \hline
\textbf{G} & - & \textbf{1} & \textbf{24} & \textbf{40} & \textbf{35} \\ \hline
RO & -1 & - & 25 & 36 & 42 \\ \hline
C-nu & -24 & -25 & - & 8 & 23 \\ \hline
B-g & -40 & -36 & -8 & - & 3 \\ \hline
$\epsilon$G & -35 & -42 & -23 & -3 & - \\ \hline
\end{tabular}
~\begin{tabular}{ | l | c | c | c | c | c |  }
\hline
$\downarrow$ vs $\rightarrow$ & G & RO & C-nu & B-g & $\epsilon$G \\ \hline
\textbf{G} & - & \textbf{1} & \textbf{14} & \textbf{8} & \textbf{15} \\ \hline
RO & -1 & - & 12 & 5 & 16 \\ \hline
C-nu & -14 & -12 & - & -12 & 5 \\ \hline
B-g & -8 & -5 & 12 & - & 10 \\ \hline
$\epsilon$G & -15 & -16 & -5 & -10 & - \\ \hline
\end{tabular}

\vspace{0.1cm}
(d) $PV_{OAA} \leq 0.2$ (left, 133 datasets), $PV_{OAA} \leq 0.05$ (right, 28 datasets)
\vspace{0.4cm}

\begin{tabular}{ | l | c | c | c | c | c |  }
\hline
$\downarrow$ vs $\rightarrow$ & G & RO & C-nu & B-g & $\epsilon$G \\ \hline
\textbf{G} & - & \textbf{2} & \textbf{8} & \textbf{5} & \textbf{12} \\ \hline
RO & -2 & - & 8 & 4 & 12 \\ \hline
C-nu & -8 & -8 & - & -1 & 12 \\ \hline
B-g & -5 & -4 & 1 & - & 11 \\ \hline
$\epsilon$G & -12 & -12 & -12 & -11 & - \\ \hline
\end{tabular}

\vspace{0.1cm}
(e) $n \geq 10\,000$ and $PV_{OAA} \leq 0.1$ (13 datasets)

\end{table}

\section{Discussion and Takeaways}
\label{sec:discussion}

In this paper, we presented an evaluation of practical contextual bandit algorithms on a large collection of supervised learning datasets with simulated bandit feedback.
We find that a worst-case theoretical robustness forces several common methods to often over-explore, damaging their empirical performance, and strategies that limit (RegCB and Cover-NU) or simply forgo (Greedy) explicit exploration dominate the field.
For practitioners, our study also provides a reference for practical implementations, while stressing the importance of loss estimation and other design choices such as how to encode observed feedback.

\paragraph{Guidelines for practitioners.}
We now summarize some practical guidelines that come out of our empirical study:
\begin{itemize}
	\item Methods relying on modeling assumptions on the data distribution such as RegCB are often preferred in practice, and even Greedy can work well (see, \eg, Table~\ref{table:win_loss_diff}).
	They tend to dominate more robust approaches such as Cover-NU even more prominently on larger datasets,
	or on datasets where prediction is easy, \eg, due to low noise (see Table~\ref{table:breakdown}).
	While it may be difficult to assess such favorable conditions of the data in advance,
	practitioners may use specific domain knowledge
	to design better feature representations for prediction,
	which may in turn improve exploration for these methods.
	\item Uniform exploration hurts empirical performance in most cases (see, \eg, the poor performance of~$\epsilon$-greedy and Cover-u in Table~\ref{table:win_loss_multiclass} of Appendix~\ref{sec:evaluation_appx}). Nevertheless, it may be necessary on the hardest datasets,
	and may be crucial if one needs to perform off-policy counterfactual evaluation (see Figures~\ref{fig:cfe} and~\ref{fig:cfe_appx}).
	\item Loss estimation is an essential component in many CB algorithms for good practical performance, and DR should be preferred over IPS. For methods based on reduction to off-policy learning, such as~$\epsilon$-Greedy and Bagging, the IWR reduction is typically best, in addition to providing computational benefits (see Table~\ref{table:reductions}).
	\item From our early experiments, we found randomization on tied choices of actions to always be useful. For instance, it avoids odd behavior which may arise from deterministic, implementation-specific biases (\eg, always favoring one specific action over the others).
	\item The choice of cost encodings makes a big difference in practice and should be carefully considered when designing a contextual bandit system, even when loss estimation techniques such as DR are used. For binary outcomes, -1/0 is a good default choice of encoding in the common situation where the observed loss is often 1 (see Table~\ref{table:encoding}).
	\item Modeling choices and encodings sometimes provide pessimistic initial estimates that can hurt initial exploration on some problems, particularly for Greedy and RegCB-optimistic.
	Random tie-breaking as well as using a shared additive baseline can help mitigate this issue (see Section~\ref{sub:baseline}).
	\item The hyperparameters highlighted in Appendix~\ref{sub:hyperparams_appx} obtained on our datasets may be good default choices in practice. Nevertheless, these may need to be balanced in order to address other conflicting requirements in real-world settings, such as non-stationarity or the ability to run offline experiments.
\end{itemize}

\paragraph{Open questions for theoreticians.}
Our study raises some questions of interest for theoretical research on contextual bandits.
The good performance of greedy methods calls for a better understanding of greedy methods,
building upon the work of~\citet{bastani2017exploiting,kannan2018smoothed},
as well as methods that are more robust to more difficult datasets while adapting to such favorable
scenarios, such as when the context distribution has enough diversity.
A related question is that of adaptivity to easy datasets for which the optimal policy has small loss,
an open problem pointed out by~\citet{agarwal2017open} in the form of ``first-order'' regret bounds.
While methods satisfying such regret bounds have now been developed theoretically~\citep{allen2018make},
these methods are currently not computationally efficient, and obtaining efficient methods based on optimization oracles
remains an important open problem.
We also note that while our experiments are based on online optimization oracles,
most analyzed versions of the algorithms rely on solving the full optimization problems;
it would be interesting to better understand the behavior of online variants,
and to characterize the implicit exploration effect for the greedy method.

\paragraph{Limitations of the study.}
Our study is primarily concerned with prediction performance, while
real world applications often additionally consider the value of
counterfactual evaluation for offline policy evaluation.  

A key limitation of our study is that it is concerned with stationary
datasets.  Many real-world contextual bandit applications involve
nonstationary datasources.  This limitation is simply due to the nature of
readily available public datasets.  The lack of public CB datasets as
well as challenges in counterfactual evaluation of CB algorithms make
a more realistic study challenging, but we hope that an emergence of
platforms~\citep{agarwal2016multiworld,jamieson2015next} to easily
deploy CB algorithms will enable studies with real CB datasets in the
future.
Furthermore, our setup does not cover some other difficulties which may often arise in practical CB systems, such as delayed/batched rewards, varying action sets, or heuristics for selecting subsets of actions in cases where the total number of actions is prohibitively large; an extension of our empirical study to such settings could also prove beneficial for practitioners.

\bibliography{full,bibli}

\clearpage
\appendix
\renewcommand{\theHsection}{A\arabic{section}}

\section{Algorithm Details}
\label{sec:alg_details}

This section provides more details on the implementations of online oracles used in our algorithms, as described in Sections~\ref{sec:setup} and~\ref{sec:algorithms}.

\paragraph{Online regression.}
Since our work considers linear models, we assume that regressors are of the form:
\[
f(x, a) = f_\theta(x, a) = \theta^\top \Phi(x, a).
\]
The basic form of online regression updates is a standard online gradient descent update on the squared loss, given below:
\begin{algorithm}[h]
\caption{Online regression update}
\label{alg:reg_update}
\textbf{Input}: step-size~$\eta$.

$\verb+reg_update+(f_\theta, (x, a, y, \omega))$:
\begin{algorithmic}
  \STATE $\theta := \theta - \eta \omega (\theta^\top \Phi(x,a) - y) \Phi(x, a)$
\end{algorithmic}
\end{algorithm}

In practice, our experiments in VW use more adaptive versions of online gradient descent that make the choice of step-size more robust to variabilities in the features~$\Phi(x, a)$, which may be quite different in different datasets in our collection.
More precisely, the learning rate~$\eta$ is replaced by a diagonal adaptive preconditioning matrix that also provides some invariance to the scale of features, a method known as normalized adaptive gradient~\citep[NAG,][]{ross2013normalized} that extends AdaGrad~\citep{duchi2011adaptive}.
Furthermore, in order to better handle the importance weighted examples in online updates, VW uses the approach of~\citet{karampatziakis2011online}, which provides more consistent updates w.r.t.~the importance weights using integration of a continuous-time ODE.

\paragraph{Cost-sensitive classification.}
For the online CSC oracle, we use a reduction to regression, by feeding~$K$ online regression examples to the online regression oracle, one for each action, with the corresponding cost.
This is shown in Algorithm~\ref{alg:csc_update}.
In order to make explicit the fact that a policy~$\pi$ is implemented using a regressor~$f$, we denote the policy by~$\pi_f$, and show the corresponding implementation of the policy.
\begin{algorithm}[h]
\caption{CSC implementation and online update}
\label{alg:csc_update}
\textbf{Input}: number of actions~$K$.

$\verb+call+(\pi_f, x)$: {\it (implementation of the call~$\pi_f(x)$)}
\begin{algorithmic}
	\STATE \textbf{return} $\arg\min_a f(x, a)$;
\end{algorithmic}

$\verb+csc_update+(\pi_f, (x, c))$:
\begin{algorithmic}
  	\FOR{$a = 1, \ldots, K$}
  		\STATE $\verb+reg_update+(f, (x, a, c(a)))$;
  	\ENDFOR
\end{algorithmic}
\end{algorithm}

\paragraph{Loss estimators (IPS and DR).}
The IPS and DR loss estimators, which implement the~$\verb+estimator+$ routine in Section~\ref{sec:algorithms}, are shown in Algorithm~\ref{alg:loss_estimators}.
For convenience, the loss estimator~$\hat \ell$, which is considered a global variable meant to be trained on all observed samples~$(x_t, a_t, \ell_t(a_t))$, is updated in the call to~$\verb+dr_estimator+$, which we assume is only called once at each iteration~$t$.
\begin{algorithm}[h]
\caption{IPS and DR Loss estimators}
\label{alg:loss_estimators}
\textbf{Global state}: loss regressor for DR~$\hat \ell$.

\vspace{0.2cm}
$\verb+ips_estimator+(x_t, a_t, y_t, p_t)$:
\begin{algorithmic}
	\STATE $\hat{\ell}_t(a) := \frac{y_t}{p_t} \1\{a = a_t\}$;
	\STATE \textbf{return} $\hat \ell_t$;
\end{algorithmic}

\vspace{0.2cm}
$\verb+dr_estimator+(x_t, a_t, y_t, p_t)$:
\begin{algorithmic}
	\STATE $\verb+reg_update+(\hat \ell, (x_t, a_t, y_t))$;
	\STATE $\hat{\ell}_t(a) := \frac{y_t - \hat{\ell}(x_t, a_t)}{p_t} \1\{a = a_t\} + \hat{\ell}(x_t, a)$;
	\STATE \textbf{return} $\hat \ell_t$;
\end{algorithmic}

\end{algorithm}

\paragraph{Off-policy learning updates.}
Finally, the online off-policy learning updates, which implement the~$\verb+opl_update+$ routine in Section~\ref{sec:algorithms}, are given below in Algorithm~\ref{alg:opl_updates}.
For IWR, we write the policy as~$\pi_f(x) = \arg\min_a f(x, a)$, where~$f$ denotes the underlying regressor.
\begin{algorithm}[h]
\caption{Off-policy learning updates}
\label{alg:opl_updates}
$\verb+ips_opl_update+(\pi, (x_t, a_t, y_t, p_t))$:
\begin{algorithmic}
	\STATE $\hat \ell_t := \verb+ips_estimator+(x_t, a_t, y_t, p_t)$;
	\STATE $\verb+csc_oracle+(\pi, (x_t, \hat \ell_t))$;
\end{algorithmic}

\vspace{0.2cm}
$\verb+dr_opl_update+(\pi, (x_t, a_t, y_t, p_t))$:
\begin{algorithmic}
	\STATE $\hat \ell_t := \verb+dr_estimator+(x_t, a_t, y_t, p_t)$;
	\STATE $\verb+csc_oracle+(\pi, (x_t, \hat \ell_t))$;
\end{algorithmic}

\vspace{0.2cm}
$\verb+iwr_opl_update+(\pi_f, (x_t, a_t, y_t, p_t))$:
\begin{algorithmic}
	\STATE $\verb+reg_oracle+(f, (x_t, a_t, y_t, 1/p_t))$;
\end{algorithmic}
\end{algorithm}

\section{Datasets}
\label{sec:datasets}

This section gives some details on the cost-sensitive classification datasets considered in our study.

\paragraph{Multiclass classification datasets.}
We consider 516 multiclass datasets\footnote{We note that previous versions of this paper considered 525 such datasets, but we later discovered that 9 of them (ids $8, 189, 197, 209, 223, 227, 287, 294, 298$) were regression datasets poorly converted to multiclass, and decided to discard them.} from the \url{openml.org} platform, including among others, medical,
gene expression, text, sensory or synthetic data. Table~\ref{table:data_multiclass} provides some statistics
about these datasets.
These also include the 8 classification datasets considered in~\citep{foster2018practical} from the UCI database.
The full list of datasets is given below.

\begin{table}[h]
\caption{Statistics on number of multiclass datasets by number of examples,
actions and unique features, as well as by progressive validation 0-1 loss for the supervised one-against-all online classifier, in our collection of 516 multiclass datasets.}
\label{table:data_multiclass}
\small
\center
\begin{tabular}{|c|c|}
\hline
actions & \# \\ \hline
2 & 402 \\ \hline
3-9 & 71 \\ \hline
10+ & 43 \\ \hline
\end{tabular}
~
\begin{tabular}{|c|c|}
\hline
examples & \# \\ \hline
$\leq$ $10^2$ & 94 \\ \hline
$10^2$-$10^3$ & 268 \\ \hline
$10^3$-$10^5$ & 126 \\ \hline
$> 10^5$ & 28 \\ \hline
\end{tabular}
~
\begin{tabular}{|c|c|}
\hline
features & \# \\ \hline
$\leq$ $50$ & 384 \\ \hline
51-100 & 34 \\ \hline
101-1000 & 17 \\ \hline
1000+ & 81 \\ \hline
\end{tabular}
~
\begin{tabular}{|c|c|}
\hline
$PV_{OAA}$ & \# \\ \hline
$\leq$ $0.01$ & 10 \\ \hline
$(0.01, 0.1]$ & 87 \\ \hline
$(0.1, 0.2]$ & 102 \\ \hline
$(0.2, 0.5]$ & 271 \\ \hline
$> 0.5$ & 46 \\ \hline
\end{tabular}
\end{table}

\paragraph{Multilabel classification datasets.}
We consider 5 multilabel datasets from the LibSVM website\footnote{\url{https://www.csie.ntu.edu.tw/~cjlin/libsvmtools/datasets/multilabel.html}},
listed in Table~\ref{table:data_multilabel}.

\begin{table}[h]
\caption{List of multilabel datasets.}
\label{table:data_multilabel}
\center
\begin{tabular}{|c|c|c|c|c|}
\hline
Dataset & \# examples & \# features & \# actions & $PV_{OAA}$ \\ \hline
mediamill & 30,993 & 120 & 101 & 0.1664 \\ \hline
rcv1 & 23,149 & 47,236 & 103 & 0.0446 \\ \hline
scene & 1,211 & 294 & 6 & 0.0066\\ \hline
tmc & 21,519 & 30,438 & 22 & 0.1661 \\ \hline
yeast & 1,500 & 103 & 14 & 0.2553 \\ \hline
\end{tabular}
\end{table}

\paragraph{Cost-sensitive classification datasets.}
For more general real-valued costs in~$[0,1]$, we use a modification of the multilabel RCV1 dataset
introduced in~\citep{krishnamurthy2017active}.
Each example consists of a news article labeled with the topics it belongs to, in a collection of 103 topics.
Instead of fixing the cost to 1 for incorrect topics, the cost is defined as the tree distance to
the set of correct topics in a topic hierarchy.

We also include the learning-to-rank datasets considered in~\citep{foster2018practical},
where we limit the number of documents (actions) per query, and consider all the training folds.
We convert relevance scores to losses in~$\{0,0.25,0.5,0.75,1\}$, with 0 indicating a perfectly relevant document,
and 1 an irrelevant one.
The datasets considered are the Microsoft Learning to Rank dataset, variant MSLR-30K at \url{https://www.microsoft.com/en-us/research/project/mslr/},
and the Yahoo! Learning to Rank Challenge V2.0, variant C14B at \url{https://webscope.sandbox.yahoo.com/catalog.php?datatype=c}.
Details are shown in Table~\ref{tab:data_ltr}.
We note that for these datasets we consider action-dependent features, with a fixed parameter vector
for all documents.

\begin{table}[h]
	\caption{Learning to rank datasets.}
	\label{tab:data_ltr}
\center
\begin{tabular}{|c|c|c|c|c|}
\hline
Dataset & \# examples & \# features & max \# documents & $PV_{OAA}$ \\ \hline
MSLR-30K & 31,531 & 136 & 10 & 0.7892 \\ \hline
Yahoo & 36,251 & 415 & 6 & 0.5876 \\ \hline
\end{tabular}
\end{table}

\paragraph{List of multiclass datasets.}
The datasets we used can be accessed at \url{https://www.openml.org/d/<id>}, with \verb+id+ in the following list:

3, 6, 10, 11, 12, 14, 16, 18, 20, 21, 22, 23, 26, 28, 30, 31, 32, 36, 37, 39, 40, 41, 43, 44, 46, 48, 50, 53, 54, 59, 60, 61, 62, 150, 151, 153, 154, 155, 156, 157, 158, 159, 160, 161, 162, 180, 181, 182, 183, 184, 187, 273, 275, 276, 277, 278, 279, 285, 292, 293, 300, 307, 310, 312, 313, 329, 333, 334, 335, 336, 337, 338, 339, 343, 346, 351, 354, 357, 375, 377, 383, 384, 385, 386, 387, 388, 389, 390, 391, 392, 393, 394, 395, 396, 397, 398, 399, 400, 401, 444, 446, 448, 450, 457, 458, 459, 461, 462, 463, 464, 465, 467, 468, 469, 472, 475, 476, 477, 478, 479, 480, 554, 679, 682, 683, 685, 694, 713, 714, 715, 716, 717, 718, 719, 720, 721, 722, 723, 724, 725, 726, 727, 728, 729, 730, 731, 732, 733, 734, 735, 736, 737, 740, 741, 742, 743, 744, 745, 746, 747, 748, 749, 750, 751, 752, 753, 754, 755, 756, 758, 759, 761, 762, 763, 764, 765, 766, 767, 768, 769, 770, 771, 772, 773, 774, 775, 776, 777, 778, 779, 780, 782, 783, 784, 785, 787, 788, 789, 790, 791, 792, 793, 794, 795, 796, 797, 799, 800, 801, 803, 804, 805, 806, 807, 808, 811, 812, 813, 814, 815, 816, 817, 818, 819, 820, 821, 822, 823, 824, 825, 826, 827, 828, 829, 830, 832, 833, 834, 835, 836, 837, 838, 841, 843, 845, 846, 847, 848, 849, 850, 851, 853, 855, 857, 859, 860, 862, 863, 864, 865, 866, 867, 868, 869, 870, 871, 872, 873, 874, 875, 876, 877, 878, 879, 880, 881, 882, 884, 885, 886, 888, 891, 892, 893, 894, 895, 896, 900, 901, 902, 903, 904, 905, 906, 907, 908, 909, 910, 911, 912, 913, 914, 915, 916, 917, 918, 919, 920, 921, 922, 923, 924, 925, 926, 927, 928, 929, 931, 932, 933, 934, 935, 936, 937, 938, 941, 942, 943, 945, 946, 947, 948, 949, 950, 951, 952, 953, 954, 955, 956, 958, 959, 962, 964, 965, 969, 970, 971, 973, 974, 976, 977, 978, 979, 980, 983, 987, 988, 991, 994, 995, 996, 997, 1004, 1005, 1006, 1009, 1011, 1012, 1013, 1014, 1015, 1016, 1019, 1020, 1021, 1022, 1025, 1026, 1036, 1038, 1040, 1041, 1043, 1044, 1045, 1046, 1048, 1049, 1050, 1054, 1055, 1056, 1059, 1060, 1061, 1062, 1063, 1064, 1065, 1066, 1067, 1068, 1069, 1071, 1073, 1075, 1077, 1078, 1079, 1080, 1081, 1082, 1083, 1084, 1085, 1086, 1087, 1088, 1100, 1104, 1106, 1107, 1110, 1113, 1115, 1116, 1117, 1120, 1121, 1122, 1123, 1124, 1125, 1126, 1127, 1128, 1129, 1130, 1131, 1132, 1133, 1135, 1136, 1137, 1138, 1139, 1140, 1141, 1142, 1143, 1144, 1145, 1146, 1147, 1148, 1149, 1150, 1151, 1152, 1153, 1154, 1155, 1156, 1157, 1158, 1159, 1160, 1161, 1162, 1163, 1164, 1165, 1166, 1169, 1216, 1217, 1218, 1233, 1235, 1236, 1237, 1238, 1241, 1242, 1412, 1413, 1441, 1442, 1443, 1444, 1449, 1451, 1453, 1454, 1455, 1457, 1459, 1460, 1464, 1467, 1470, 1471, 1472, 1473, 1475, 1481, 1482, 1483, 1486, 1487, 1488, 1489, 1496, 1498, 1590.

\section{Evaluation Details}
\label{sec:evaluation_appx}

\subsection{Algorithms and Hyperparameters}
\label{sub:hyperparams_appx}

We ran each method on every dataset with the following hyperparameters:
\begin{itemize}
	\item algorithm-specific \emph{hyperparameters}, shown in Table~\ref{table:algos}.
	\item 9 choices of \emph{learning rates}, on a logarithmic grid from 0.001 to 10 (see Section~\ref{sub:exp_setup}).
	\item 3 choices of \emph{reductions}: IPS, DR and IWR (see Section~\ref{sub:reductions}).
	Note that these mainly apply to methods that reduce to off-policy optimization (\ie, $\epsilon$-Greedy and Bag/BTS),
	and to some extent, methods that reduce to cost-sensitive classification (\ie, cover and active $\epsilon$-greedy,
	though the IWR reduction is heuristic in this case).
	Both RegCB variants directly reduce to~regression.
	\item 3 choices of loss \emph{encodings}: 0/1, -1/0 and 9/10 (see Eq.~\eqref{eq:loss_enc_def}).
	0/1 and -1/0 encodings are typically a design choice, while the experiments with 9/10 are aimed at assessing
	some robustness to loss range.
\end{itemize}

\begin{table}[tb]
\caption{Choices of hyperparameters and reduction for each method.
Fixed choices of hyperparameters for -1/0 encodings are in \textbf{bold}.
These were obtained for each method with an instant-runoff voting mechanism
on 200 of the multiclass datasets with -1/0 encoding,
where each dataset ranks hyperparameter choices according to the difference between significant wins
and losses against all other choices (the vote of each dataset is divided by the number of tied choices ranked first). Table~\ref{table:hyperparams2} shows optimized choices of hyperparameters for different encoding settings used in our study.
}
\label{table:algos}

\center
\begin{tabular}{ | c | c | c | c | }
\hline
Name & Method & Hyperparameters & Reduction \\ \hline
G & Greedy & - & \textbf{IWR} \\ \hline
R/RO & RegCB-elim/RegCB-opt & $C_0 \in 10^{-\{1, 2, \textbf{3}\}}$ & - \\ \hline
\multirow{2}{*}{C-nu} & \multirow{2}{*}{Cover-NU} & $N \in \{\textbf{4}, 8, 16\}$ & \multirow{2}{*}{IPS/\textbf{DR}}\\
	& & $\psi \in \{0.01, \textbf{0.1}, 1\}$ & \\ \hline
\multirow{2}{*}{C-u} & \multirow{2}{*}{Cover} & $N \in \{\textbf{4}, 8, 16\}$ & \multirow{2}{*}{\textbf{IPS}/DR}\\
	& & $\psi \in \{0.01, \textbf{0.1}, 1\}$ & \\ \hline
B/B-g & Bag/Bag-greedy & $N \in \{\textbf{4}, 8, 16\}$ & IPS/DR/\textbf{IWR} \\ \hline
$\epsilon$G & $\epsilon$-greedy & $\epsilon \in \{\textbf{0.02}, 0.05, 0.1\}$ & IPS/DR/\textbf{IWR} \\ \hline
\multirow{2}{*}{A} & \multirow{2}{*}{active $\epsilon$-greedy} & $\epsilon \in \{\textbf{0.02}, 1\}$ & \multirow{2}{*}{IPS/DR/\textbf{IWR}} \\
 	& & $C_0 \in 10^{-\{2, 4, \textbf{6}\}}$ & \\ \hline

\end{tabular}
\end{table}

\begin{table}
\caption{Optimized choices of hyperparameters for different encoding settings,
obtained using the voting mechanism described in Table~\ref{table:algos}:
-1/0 (same as bold choices in Table~\ref{table:algos}, used in Tables~\ref{table:win_loss_diff}(top left), \ref{table:cost_sensitive}ce, \ref{table:breakdown}, \ref{table:win_loss_multiclass}a, \ref{table:win_loss_multilab}a, \ref{table:win_loss_uci}a and in the figures);
0/1 (used in Tables~\ref{table:win_loss_diff}(bottom left), \ref{table:cost_sensitive}abd, \ref{table:win_loss_multiclass}b, \ref{table:win_loss_multilab}b, \ref{table:win_loss_uci}b, \ref{table:rcv1cs_all}).
}
\label{table:hyperparams2}

\small
\center
\begin{tabular}{|c|c|c|}
\hline
Algorithm & \textbf{-1/0} & \textbf{0/1} \\ \hline
G & - & - \\ \hline
R/RO & $C_0 = 10^{-3}$ & $C_0 = 10^{-3}$ \\ \hline
C-nu & $N = 4, \psi = 0.1$, DR & $N = 4, \psi = 0.01$, DR \\ \hline
C-u & $N = 4, \psi = 0.1$, IPS & $N = 4, \psi = 0.1$, DR \\ \hline
B & $N = 4$, IWR & $N = 16$, IWR \\ \hline
B-g & $N = 4$, IWR & $N = 8$, IWR \\ \hline
$\epsilon$G & $\epsilon = 0.02$, IWR & $\epsilon = 0.02$, IWR \\ \hline
A & $\epsilon = 0.02, C_0 = 10^{-6}$, IWR & $\epsilon = 0.02, C_0 = 10^{-6}$, IWR \\ \hline
\end{tabular}
\end{table}

\subsection{Additional Evaluation Results}
\label{sub:results_appx}

This sections provides additional experimental results, and more detailed win/loss statistics
for tables in the main paper, showing both significant wins and significant losses,
rather than just their difference.

\paragraph{Extended tables.}
Tables~\ref{table:win_loss_multiclass} and~\ref{table:win_loss_multilab}
are extended versions of Table~\ref{table:win_loss_diff}, showing both significant wins and loss,
more methods, and separate statistics for multiclass and multilabel datasets.
In particular, we can see that both variants of RegCB become even more competitive against all other methods
when using 0/1 encodings.
Table~\ref{table:rcv1cs_all} extends Table~\ref{table:cost_sensitive}(a) with additional methods.
Table~\ref{table:reductions_win_loss} is a more detailed win/loss version of Table~\ref{table:reductions},
and additionally shows statistics for 0/1 encodings.

We also show separate statistics in Table~\ref{table:win_loss_uci} for the 8 datasets from the UCI repository considered in~\citep{foster2018practical},
which highlight that Greedy can outperform RegCB on some of these datasets,
and that the optimistic variant of RegCB is often superior to the elimination variant.
We note that our experimental setup is quite different from~\citet{foster2018practical},
who consider batch learning on an doubling epoch schedule,
which might explain some of the differences in the results.

\begin{table}[tb]
\caption{\emph{Statistically significant} wins / losses of all methods on the 317 held-out multiclass classification datasets.
Hyperparameters are fixed as given in Table~\ref{table:hyperparams2}.
}
\label{table:win_loss_multiclass}
\tiny
\center
\begin{tabular}{ | l | c | c | c | c | c | c | c | c | c |  }
\hline
$\downarrow$ vs $\rightarrow$ & G & R & RO & C-nu & B & B-g & $\epsilon$G & C-u & A \\ \hline
G & - & 18 / 24 & 6 / 16 & 39 / 30 & 76 / 15 & 63 / 16 & 58 / 6 & 153 / 10 & 33 / 9 \\ \hline
R & 24 / 18 & - & 15 / 19 & 47 / 27 & 73 / 15 & 60 / 17 & 69 / 8 & 155 / 8 & 48 / 11 \\ \hline
RO & 16 / 6 & 19 / 15 & - & 46 / 18 & 76 / 11 & 59 / 11 & 70 / 2 & 162 / 5 & 46 / 4 \\ \hline
C-nu & 30 / 39 & 27 / 47 & 18 / 46 & - & 63 / 24 & 48 / 29 & 71 / 18 & 153 / 9 & 47 / 26 \\ \hline
B & 15 / 76 & 15 / 73 & 11 / 76 & 24 / 63 & - & 9 / 31 & 45 / 41 & 121 / 16 & 26 / 52 \\ \hline
B-g & 16 / 63 & 17 / 60 & 11 / 59 & 29 / 48 & 31 / 9 & - & 48 / 29 & 125 / 10 & 27 / 39 \\ \hline
$\epsilon$G & 6 / 58 & 8 / 69 & 2 / 70 & 18 / 71 & 41 / 45 & 29 / 48 & - & 121 / 14 & 2 / 35 \\ \hline
C-u & 10 / 153 & 8 / 155 & 5 / 162 & 9 / 153 & 16 / 121 & 10 / 125 & 14 / 121 & - & 9 / 147 \\ \hline
A & 9 / 33 & 11 / 48 & 4 / 46 & 26 / 47 & 52 / 26 & 39 / 27 & 35 / 2 & 147 / 9 & - \\ \hline
\end{tabular}

(a) -1/0 encoding
\vspace{0.3cm}

\begin{tabular}{ | l | c | c | c | c | c | c | c | c | c |  }
\hline
$\downarrow$ vs $\rightarrow$ & G & R & RO & C-nu & B & B-g & $\epsilon$G & C-u & A \\ \hline
G & - & 26 / 64 & 5 / 67 & 34 / 49 & 74 / 36 & 71 / 36 & 66 / 19 & 144 / 29 & 35 / 17 \\ \hline
R & 64 / 26 & - & 14 / 34 & 40 / 21 & 86 / 9 & 82 / 12 & 105 / 16 & 165 / 3 & 76 / 19 \\ \hline
RO & 67 / 5 & 34 / 14 & - & 57 / 12 & 108 / 7 & 104 / 8 & 117 / 2 & 168 / 3 & 88 / 3 \\ \hline
C-nu & 49 / 34 & 21 / 40 & 12 / 57 & - & 82 / 25 & 72 / 26 & 94 / 24 & 164 / 5 & 61 / 30 \\ \hline
B & 36 / 74 & 9 / 86 & 7 / 108 & 25 / 82 & - & 21 / 31 & 57 / 48 & 124 / 14 & 36 / 66 \\ \hline
B-g & 36 / 71 & 12 / 82 & 8 / 104 & 26 / 72 & 31 / 21 & - & 61 / 46 & 122 / 18 & 39 / 54 \\ \hline
$\epsilon$G & 19 / 66 & 16 / 105 & 2 / 117 & 24 / 94 & 48 / 57 & 46 / 61 & - & 116 / 28 & 2 / 41 \\ \hline
C-u & 29 / 144 & 3 / 165 & 3 / 168 & 5 / 164 & 14 / 124 & 18 / 122 & 28 / 116 & - & 22 / 144 \\ \hline
A & 17 / 35 & 19 / 76 & 3 / 88 & 30 / 61 & 66 / 36 & 54 / 39 & 41 / 2 & 144 / 22 & - \\ \hline
\end{tabular}

(b) 0/1 encoding
\end{table}

\begin{table}[tb]
\small
\center
\caption{\emph{Statistically significant} wins / losses of all methods on the 5 multilabel classification datasets.
Hyperparameters are fixed as given in Table~\ref{table:hyperparams2}.
}
\label{table:win_loss_multilab}
\begin{tabular}{ | l | c | c | c | c | c | c | c | c | c |  }
\hline
$\downarrow$ vs $\rightarrow$ & G & R & RO & C-nu & B & B-g & $\epsilon$G & C-u & A \\ \hline
G & - & 2 / 2 & 1 / 0 & 3 / 1 & 2 / 2 & 3 / 2 & 2 / 1 & 4 / 1 & 2 / 1 \\ \hline
R & 2 / 2 & - & 2 / 2 & 2 / 0 & 3 / 0 & 3 / 0 & 3 / 2 & 5 / 0 & 3 / 2 \\ \hline
RO & 0 / 1 & 2 / 2 & - & 2 / 2 & 2 / 2 & 3 / 1 & 2 / 1 & 4 / 1 & 2 / 1 \\ \hline
C-nu & 1 / 3 & 0 / 2 & 2 / 2 & - & 3 / 1 & 3 / 1 & 3 / 2 & 5 / 0 & 3 / 2 \\ \hline
B & 2 / 2 & 0 / 3 & 2 / 2 & 1 / 3 & - & 1 / 1 & 3 / 2 & 5 / 0 & 3 / 2 \\ \hline
B-g & 2 / 3 & 0 / 3 & 1 / 3 & 1 / 3 & 1 / 1 & - & 2 / 3 & 4 / 1 & 2 / 3 \\ \hline
$\epsilon$G & 1 / 2 & 2 / 3 & 1 / 2 & 2 / 3 & 2 / 3 & 3 / 2 & - & 4 / 1 & 1 / 0 \\ \hline
C-u & 1 / 4 & 0 / 5 & 1 / 4 & 0 / 5 & 0 / 5 & 1 / 4 & 1 / 4 & - & 1 / 4 \\ \hline
A & 1 / 2 & 2 / 3 & 1 / 2 & 2 / 3 & 2 / 3 & 3 / 2 & 0 / 1 & 4 / 1 & - \\ \hline
\end{tabular}

(a) -1/0 encoding
\vspace{0.3cm}

\begin{tabular}{ | l | c | c | c | c | c | c | c | c | c |  }
\hline
$\downarrow$ vs $\rightarrow$ & G & R & RO & C-nu & B & B-g & $\epsilon$G & C-u & A \\ \hline
G & - & 3 / 1 & 2 / 2 & 1 / 3 & 4 / 0 & 4 / 1 & 4 / 0 & 5 / 0 & 3 / 0 \\ \hline
R & 1 / 3 & - & 0 / 3 & 1 / 4 & 2 / 2 & 2 / 3 & 2 / 2 & 4 / 1 & 2 / 2 \\ \hline
RO & 2 / 2 & 3 / 0 & - & 1 / 2 & 5 / 0 & 4 / 0 & 3 / 1 & 5 / 0 & 3 / 1 \\ \hline
C-nu & 3 / 1 & 4 / 1 & 2 / 1 & - & 4 / 1 & 3 / 1 & 4 / 0 & 4 / 1 & 4 / 1 \\ \hline
B & 0 / 4 & 2 / 2 & 0 / 5 & 1 / 4 & - & 0 / 2 & 1 / 2 & 2 / 1 & 1 / 3 \\ \hline
B-g & 1 / 4 & 3 / 2 & 0 / 4 & 1 / 3 & 2 / 0 & - & 2 / 2 & 4 / 1 & 1 / 3 \\ \hline
$\epsilon$G & 0 / 4 & 2 / 2 & 1 / 3 & 0 / 4 & 2 / 1 & 2 / 2 & - & 4 / 1 & 0 / 1 \\ \hline
C-u & 0 / 5 & 1 / 4 & 0 / 5 & 1 / 4 & 1 / 2 & 1 / 4 & 1 / 4 & - & 0 / 5 \\ \hline
A & 0 / 3 & 2 / 2 & 1 / 3 & 1 / 4 & 3 / 1 & 3 / 1 & 1 / 0 & 5 / 0 & - \\ \hline
\end{tabular}

(b) 0/1 encoding
\end{table}

\begin{table}[tb]
\small
\center
\caption{\emph{Statistically significant} wins / losses of all methods on the
8 classification datasets from the UCI repository considered in~\citep{foster2018practical}.
Hyperparameters are fixed as given in Table~\ref{table:hyperparams2}.
}
\label{table:win_loss_uci}
\begin{tabular}{ | l | c | c | c | c | c | c | c | c | c |  }
\hline
$\downarrow$ vs $\rightarrow$ & G & R & RO & C-nu & B & B-g & $\epsilon$G & C-u & A \\ \hline
\textbf{G} & - & 4 / 0 & 2 / 1 & 5 / 2 & 5 / 0 & 4 / 1 & 6 / 0 & 6 / 0 & 5 / 0 \\ \hline
R & 0 / 4 & - & 0 / 1 & 3 / 1 & 4 / 1 & 3 / 1 & 4 / 2 & 6 / 0 & 2 / 2 \\ \hline
RO & 1 / 2 & 1 / 0 & - & 4 / 1 & 5 / 0 & 5 / 0 & 5 / 0 & 6 / 0 & 3 / 0 \\ \hline
C-nu & 2 / 5 & 1 / 3 & 1 / 4 & - & 3 / 1 & 2 / 2 & 3 / 2 & 7 / 0 & 2 / 3 \\ \hline
B & 0 / 5 & 1 / 4 & 0 / 5 & 1 / 3 & - & 0 / 2 & 3 / 2 & 6 / 0 & 2 / 2 \\ \hline
B-g & 1 / 4 & 1 / 3 & 0 / 5 & 2 / 2 & 2 / 0 & - & 4 / 1 & 6 / 0 & 2 / 2 \\ \hline
$\epsilon$G & 0 / 6 & 2 / 4 & 0 / 5 & 2 / 3 & 2 / 3 & 1 / 4 & - & 6 / 0 & 0 / 2 \\ \hline
C-u & 0 / 6 & 0 / 6 & 0 / 6 & 0 / 7 & 0 / 6 & 0 / 6 & 0 / 6 & - & 0 / 6 \\ \hline
A & 0 / 5 & 2 / 2 & 0 / 3 & 3 / 2 & 2 / 2 & 2 / 2 & 2 / 0 & 6 / 0 & - \\ \hline
\end{tabular}

(a) -1/0 encoding
\vspace{0.3cm}

\begin{tabular}{ | l | c | c | c | c | c | c | c | c | c |  }
\hline
$\downarrow$ vs $\rightarrow$ & G & R & RO & C-nu & B & B-g & $\epsilon$G & C-u & A \\ \hline
G & - & 1 / 5 & 0 / 5 & 2 / 2 & 3 / 2 & 3 / 1 & 3 / 1 & 7 / 0 & 2 / 1 \\ \hline
R & 5 / 1 & - & 0 / 1 & 5 / 0 & 6 / 0 & 6 / 0 & 6 / 0 & 8 / 0 & 5 / 0 \\ \hline
\textbf{RO} & 5 / 0 & 1 / 0 & - & 5 / 0 & 6 / 0 & 6 / 0 & 7 / 0 & 7 / 0 & 6 / 0 \\ \hline
C-nu & 2 / 2 & 0 / 5 & 0 / 5 & - & 2 / 1 & 3 / 0 & 2 / 1 & 7 / 0 & 1 / 1 \\ \hline
B & 2 / 3 & 0 / 6 & 0 / 6 & 1 / 2 & - & 3 / 0 & 1 / 3 & 7 / 0 & 0 / 3 \\ \hline
B-g & 1 / 3 & 0 / 6 & 0 / 6 & 0 / 3 & 0 / 3 & - & 1 / 3 & 7 / 0 & 0 / 3 \\ \hline
$\epsilon$G & 1 / 3 & 0 / 6 & 0 / 7 & 1 / 2 & 3 / 1 & 3 / 1 & - & 7 / 0 & 0 / 1 \\ \hline
C-u & 0 / 7 & 0 / 8 & 0 / 7 & 0 / 7 & 0 / 7 & 0 / 7 & 0 / 7 & - & 0 / 7 \\ \hline
A & 1 / 2 & 0 / 5 & 0 / 6 & 1 / 1 & 3 / 0 & 3 / 0 & 1 / 0 & 7 / 0 & - \\ \hline
\end{tabular}

(b) 0/1 encoding
\end{table}

\begin{table}[tb]
\caption{Progressive validation loss for RCV1 with real-valued costs.
Same as Table~\ref{table:cost_sensitive}(a), but with all methods.
Hyperparameters are fixed as given in Table~\ref{table:hyperparams2}.
The learning rate is optimized once on the original dataset,
and we show mean and standard error based on 10 different random reshufflings of the dataset.}
\label{table:rcv1cs_all}
\small
\center
\begin{tabular}{ | c | c | c | c | c |  }
\hline
G & R & RO & C-nu & C-u \\ \hline 
0.215 $\pm$ 0.010 & 0.408 $\pm$ 0.003 & 0.225 $\pm$ 0.008 & 0.215 $\pm$ 0.006 & 0.570 $\pm$ 0.023 \\ \hline
\end{tabular}
\begin{tabular}{| c | c | c | c |}
\hline
B & B-g & $\epsilon$G & A \\ \hline
0.256 $\pm$ 0.006 & 0.251 $\pm$ 0.005 & 0.230 $\pm$ 0.009 & 0.230 $\pm$ 0.010 \\ \hline
\end{tabular}
\end{table}

\begin{table}[tb]
\caption{Impact of reductions for Bag (left) and $\epsilon$-greedy (right),
with hyperparameters optimized and encoding fixed to -1/0 or 0/1.
Extended version of Table~\ref{table:reductions}.
Each (row, column) entry shows the \emph{statistically significant} wins and losses
of row against column.}
\label{table:reductions_win_loss}

\small
\center
\begin{tabular}{ | l | c | c | c |  }
\hline
$\downarrow$ vs $\rightarrow$ & ips & dr & iwr \\ \hline
ips & - & 29 / 71 & 25 / 80 \\ \hline
dr & 71 / 29 & - & 30 / 55 \\ \hline
iwr & 80 / 25 & 55 / 30 & - \\ \hline
\end{tabular}
~~\begin{tabular}{ | l | c | c | c |  }
\hline
$\downarrow$ vs $\rightarrow$ & ips & dr & iwr \\ \hline
ips & - & 83 / 22 & 16 / 144 \\ \hline
dr & 22 / 83 & - & 9 / 159 \\ \hline
iwr & 144 / 16 & 159 / 9 & - \\ \hline
\end{tabular}

(a) -1/0 encoding
\vspace{0.3cm}

\begin{tabular}{ | l | c | c | c |  }
\hline
$\downarrow$ vs $\rightarrow$ & ips & dr & iwr \\ \hline
ips & - & 46 / 231 & 17 / 236 \\ \hline
dr & 231 / 46 & - & 33 / 95 \\ \hline
iwr & 236 / 17 & 95 / 33 & - \\ \hline
\end{tabular}
~~\begin{tabular}{ | l | c | c | c |  }
\hline
$\downarrow$ vs $\rightarrow$ & ips & dr & iwr \\ \hline
ips & - & 40 / 129 & 36 / 174 \\ \hline
dr & 129 / 40 & - & 23 / 142 \\ \hline
iwr & 174 / 36 & 142 / 23 & - \\ \hline
\end{tabular}

(b) 0/1 encoding
\end{table}

\paragraph{Varying~$C_0$ in RegCB-opt and active $\epsilon$-greedy.}
Figure~\ref{fig:ro} shows a comparison between RegCB-opt and Greedy or Cover-NU on our corpus,
for different values of~$C_0$, which controls the level of exploration through the width of confidence bounds.
Figure~\ref{fig:active} shows the improvements that the active $\epsilon$-greedy algorithm can achieve
compared to $\epsilon$-greedy, under different settings.

\begin{figure}[tb]
	\centering
	\includegraphics[width=0.30\columnwidth]{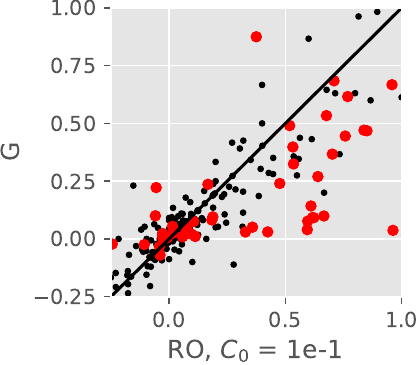}
	\includegraphics[width=0.30\columnwidth]{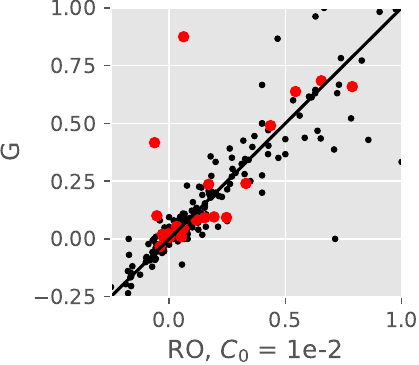}
	\includegraphics[width=0.30\columnwidth]{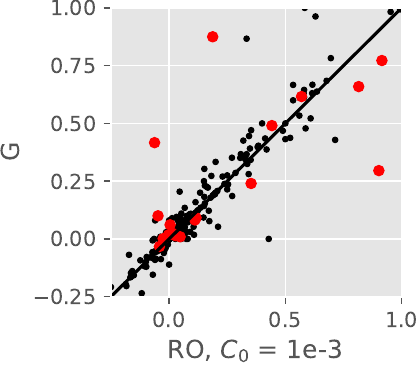}
	\includegraphics[width=0.30\columnwidth]{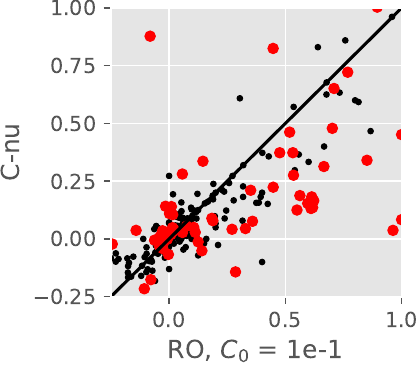}
	\includegraphics[width=0.30\columnwidth]{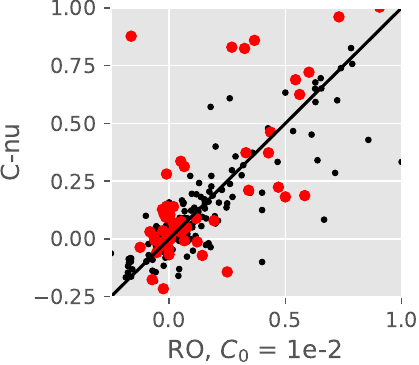}
	\includegraphics[width=0.30\columnwidth]{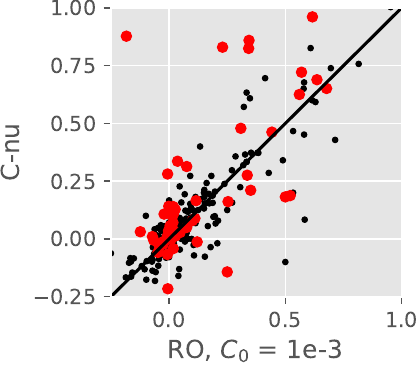}

	\caption{Comparison of RegCB-opt with Greedy (top) and Cover-NU (bottom) for different values of~$C_0$.
	Hyperparameters for Greedy and Cover-NU fixed as in Table~\ref{table:algos}. Encoding fixed to -1/0.
	The plots consider normalized loss on held-out datasets, with red points indicating significant wins.
	}
	\label{fig:ro}
\end{figure}

\begin{figure}[tb]
	\centering
	\small
	\includegraphics[width=0.30\columnwidth]{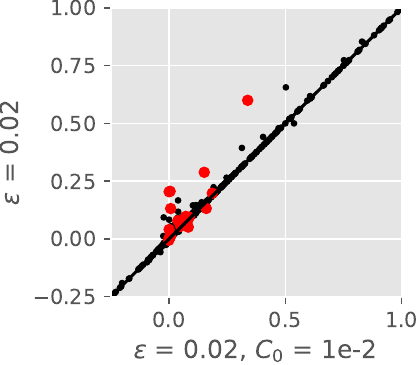}
	\includegraphics[width=0.30\columnwidth]{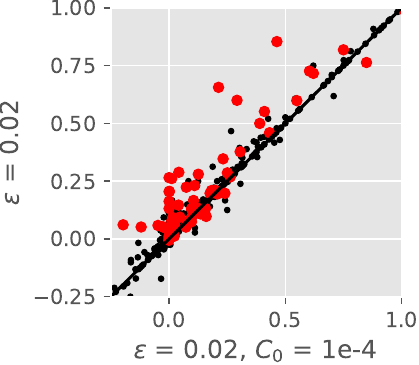}
	\includegraphics[width=0.30\columnwidth]{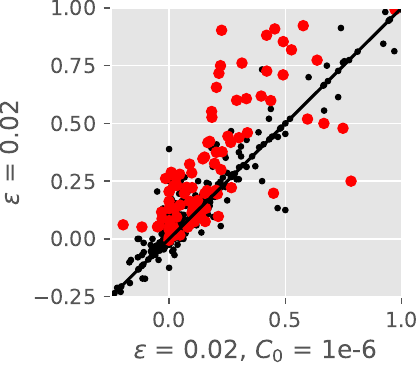}\\
	(a) DR\\ \vspace{0.2cm}
	\includegraphics[width=0.30\columnwidth]{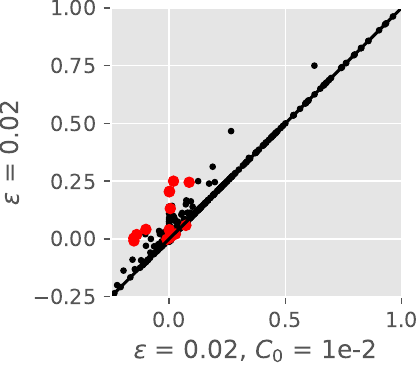}
	\includegraphics[width=0.30\columnwidth]{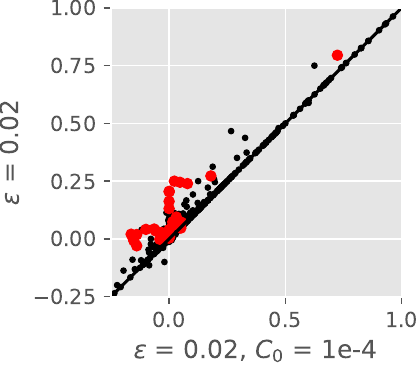}
	\includegraphics[width=0.30\columnwidth]{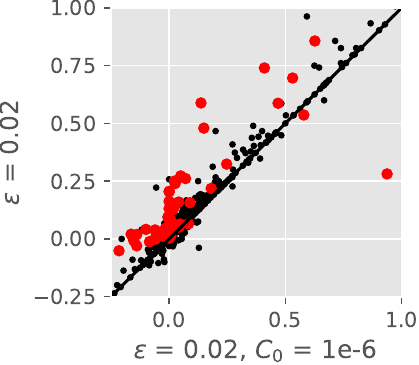}\\
	(b) \MTR{}

	\caption{Improvements to $\epsilon$-greedy from our active learning strategy. Encoding fixed to -1/0.
	The \MTR{} implementation described in Section~\ref{sub:active_impl}
	still manages to often outperform $\epsilon$-greedy,
	despite only providing an approximation to Algorithm~\ref{alg:egreedy_active}.}
	\label{fig:active}
\end{figure}

\paragraph{Counterfactual evaluation.}
Figure~\ref{fig:cfe_appx} extends Figure~\ref{fig:cfe} to include all algorithms,
and additionally shows results of using IPS estimates directly on the losses~$\ell_t(a_t)$
instead of rewards $1 - \ell_t(a_t)$, which tend to be significantly worse.

\begin{figure}[tb]
	\centering
	\begin{subfigure}[c]{.45\textwidth}
	\includegraphics[width=\textwidth]{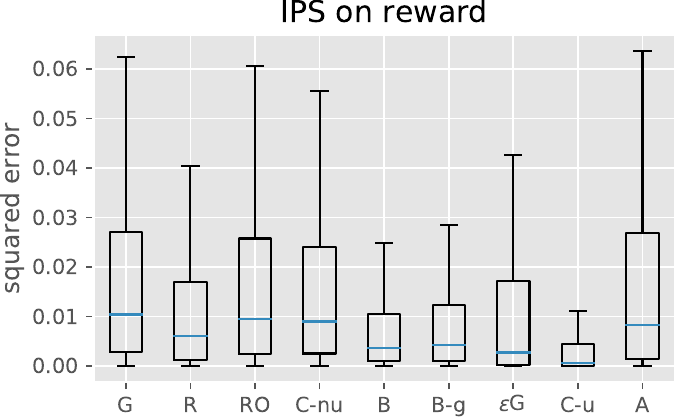}
	\caption{all multiclass datasets}
	\end{subfigure}
	\begin{subfigure}[c]{.45\textwidth}
	\includegraphics[width=\textwidth]{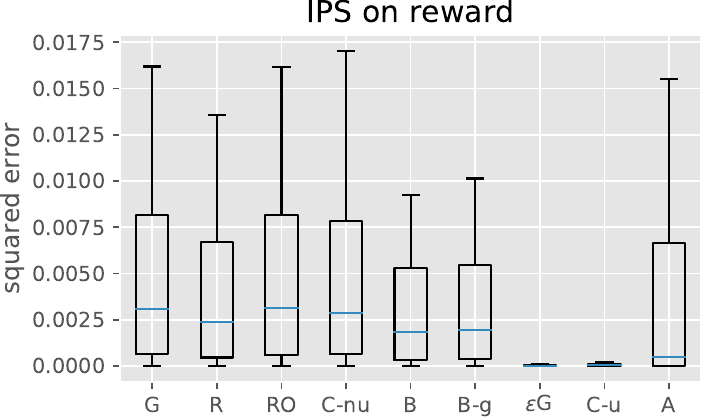}
	\caption{$n \geq 10\,000$ only}
	\end{subfigure}
	\begin{subfigure}[c]{.45\textwidth}
	\includegraphics[width=\textwidth]{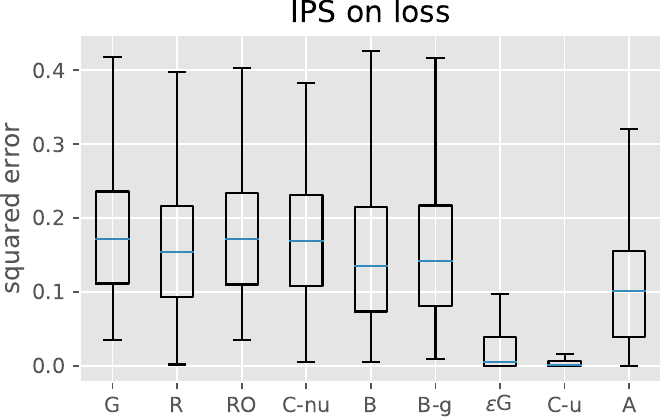}
	\caption{all multiclass datasets}
	\end{subfigure}
	\begin{subfigure}[c]{.45\textwidth}
	\includegraphics[width=\textwidth]{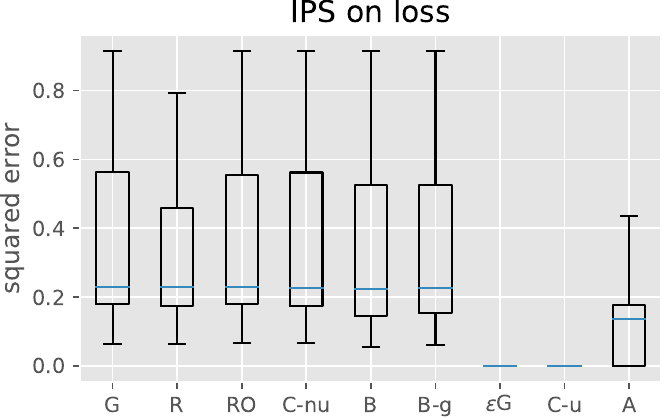}
	\caption{$n \geq 10\,000$ only}
	\end{subfigure}
	\caption{Errors of IPS counterfactual estimates for the uniform random policy using exploration
	logs collected by various algorithms on multiclass datasets (extended version of Figure~\ref{fig:cfe}).
	The boxes show quartiles (with the median shown as a blue line) of the distribution of squared errors across all multiclass datasets or only those with at least 10\,000 examples.
	The logs are obtained by running each algorithm with -1/0 encodings, fixed hyperparameters from Table~\ref{table:algos}, and the best learning rate on each dataset according to progressive validation loss.
	The top plots consider IPS with \emph{reward} estimates (as in Figure~\ref{fig:cfe}), while the bottom plots consider IPS on the \emph{loss}.}
	\label{fig:cfe_appx}
\end{figure}

\subsection{Shared \emph{baseline} parameterization}
\label{sub:baseline}

We also experimented with the use of an \emph{action-independent additive baseline} term in our loss estimators,
which can help learn better estimates with fewer samples in some situations.
In this case the regressors take the form $f(x, a) = \theta_0 + \theta_a^\top x$ (\MTR{})
or $\hat{\ell}(x, a) = \phi_0 + \phi_a^\top x$ (DR).
In order to learn the baseline term more quickly,
we propose to use a separate online update for the parameters $\theta_0$ or $\phi_0$ to regress on observed losses,
followed by an online update on the residual for the action-dependent part.
We scale the step-size of these baseline updates by the largest observed magnitude of the loss,
in order to adapt to the observed loss range for normalized updates~\citep{ross2013normalized}.

Such an additive baseline can be helpful to
quickly adapt to a constant loss estimate thanks to the separate online update.
This appears particularly useful with the -1/0 encoding, for which the initialization at 0 may give
pessimistic loss estimates which can be damaging in particular for the greedy method,
that often gets some initial exploration from an optimistic cost encoding. This can be seen
in Figure~\ref{fig:baseline}(top).
Table~\ref{table:baseline_opt} shows that optimizing over the use of \emph{baseline} on each dataset
can improve the performance of Greedy and RegCB-opt when compared to other methods such as Cover-NU.

In an online learning setting, baseline can also help to quickly reach
an unknown target range of loss estimates.
This is demonstrated in Figure~\ref{fig:baseline}(bottom),
where the addition of baseline is shown to help various methods with 9/10 encodings
on a large number of datasets. We do not evaluate RegCB for 9/10 encodings as it needs a priori known upper and lower bounds on costs.

\begin{figure}[tb]
	\centering
	\includegraphics[width=0.30\columnwidth]{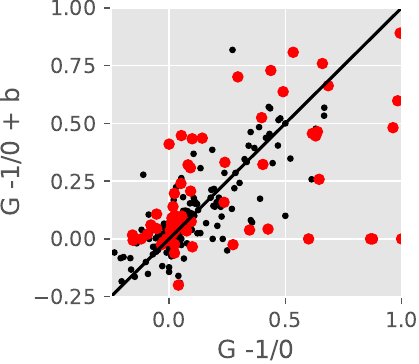}
	~\includegraphics[width=0.30\columnwidth]{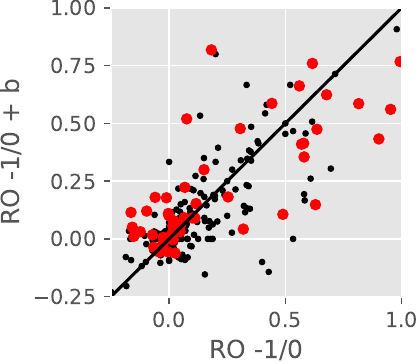}
	~\includegraphics[width=0.30\columnwidth]{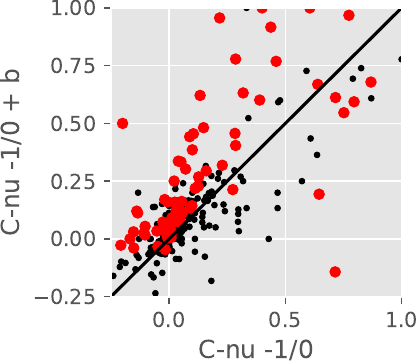} \\
	\includegraphics[width=0.30\columnwidth]{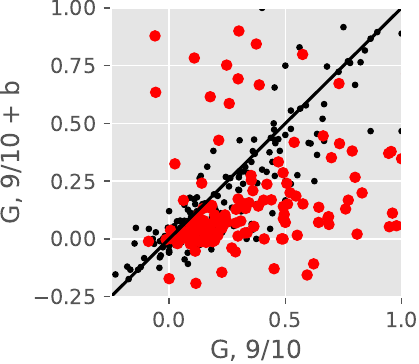}
	~\includegraphics[width=0.30\columnwidth]{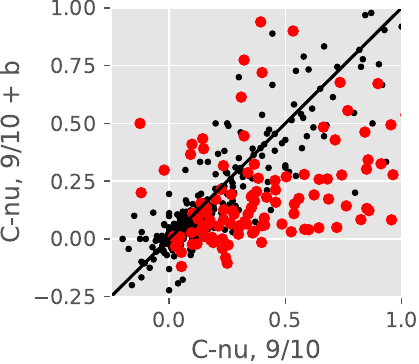}
	~\includegraphics[width=0.30\columnwidth]{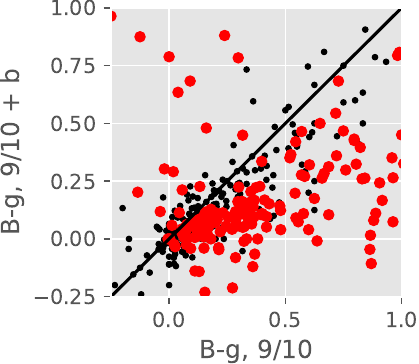}

	\caption{(top) Impact of \emph{baseline} on different algorithms with encoding fixed to -1/0;
	for Greedy and RegCB-opt, it can significantly help against pessimistic initial costs in some datasets.
	Hyperparameters fixed as in Table~\ref{table:algos}.
	(bottom) Baseline improves robustness to the range of losses.
	The plots consider normalized loss on held-out datasets, with red points indicating significant wins.}
	\label{fig:baseline}
\end{figure}

\begin{table}
\small
\center
\caption{\emph{Statistically significant} wins / losses of all methods on held-out datasets,
with -1/0 encoding and fixed hyperparameters, except for \emph{baseline}, which is optimized on each dataset together with the learning rate.
The fixed hyperparameters are shown in the table below, and were selected with the same voting
approach described in Table~\ref{table:algos}.
This optimization benefits Greedy and RegCB-opt in particular.}
\label{table:baseline_opt}

\begin{tabular}{ | l | c | c | c | c | c | c | c | c | c |  }
\hline
$\downarrow$ vs $\rightarrow$ & G & R & RO & C-nu & B & B-g & $\epsilon$G & C-u & A \\ \hline
G & - & 28 / 13 & 11 / 16 & 58 / 14 & 87 / 13 & 68 / 16 & 75 / 0 & 171 / 1 & 48 / 4 \\ \hline
R & 13 / 28 & - & 8 / 32 & 41 / 21 & 68 / 15 & 51 / 20 & 64 / 10 & 163 / 7 & 42 / 15 \\ \hline
RO & 16 / 11 & 32 / 8 & - & 63 / 11 & 89 / 10 & 66 / 10 & 82 / 2 & 180 / 1 & 61 / 3 \\ \hline
C-nu & 14 / 58 & 21 / 41 & 11 / 63 & - & 56 / 30 & 32 / 40 & 57 / 23 & 158 / 6 & 40 / 33 \\ \hline
B & 13 / 87 & 15 / 68 & 10 / 89 & 30 / 56 & - & 10 / 32 & 53 / 34 & 128 / 7 & 29 / 57 \\ \hline
B-g & 16 / 68 & 20 / 51 & 10 / 66 & 40 / 32 & 32 / 10 & - & 56 / 17 & 140 / 2 & 37 / 36 \\ \hline
$\epsilon$G & 0 / 75 & 10 / 64 & 2 / 82 & 23 / 57 & 34 / 53 & 17 / 56 & - & 126 / 10 & 3 / 42 \\ \hline
C-u & 1 / 171 & 7 / 163 & 1 / 180 & 6 / 158 & 7 / 128 & 2 / 140 & 10 / 126 & - & 5 / 159 \\ \hline
A & 4 / 48 & 15 / 42 & 3 / 61 & 33 / 40 & 57 / 29 & 36 / 37 & 42 / 3 & 159 / 5 & - \\ \hline
\end{tabular}

\vspace{0.3cm}
\begin{tabular}{|c|c|}
\hline
Algorithm & Hyperparameters \\ \hline
G &  - \\ \hline
R/RO & $C_0 = 10^{-3}$ \\ \hline
C-nu & $N = 16, \psi = 0.1$, DR \\ \hline
C-u & $N = 4, \psi = 0.1$, IPS \\ \hline
B & $N = 4$, IWR \\ \hline
B-g & $N = 4$, IWR \\ \hline
$\epsilon$G & $\epsilon = 0.02$, IWR \\ \hline
A & $\epsilon = 0.02, C_0 = 10^{-6}$, IWR \\ \hline
\end{tabular}

\end{table}

\clearpage
\section{Active $\epsilon$-greedy: Practical Algorithm and Analysis}
\label{sec:active_e_greedy_appx}

This section presents our active $\epsilon$-greedy method, a variant of $\epsilon$-greedy
that reduces the amount of uniform exploration using techniques from active learning.
Section~\ref{sub:active_impl} introduces the practical algorithm,\footnote{Our implementation is available in the following branch of Vowpal Wabbit: \url{https://github.com/albietz/vowpal_wabbit/tree/bakeoff}.} while Section~\ref{sub:active_analysis}
provides a theoretical analysis of the method, showing that it achieves a regret of~$O(T^{1/3})$
under specific favorable settings, compared to~$O(T^{2/3})$ for vanilla $\epsilon$-greedy.

\subsection{Algorithm}
\label{sub:active_impl}

\begin{algorithm}[tb]
\caption{Active $\epsilon$-greedy}
\label{alg:egreedy_active}
$\pi_1$; $\epsilon$; $C_0 > 0$.

$\verb+explore+(x_t)$:
\begin{algorithmic}
  \STATE $A_t = \{a : \verb+loss_diff+(\pi_t, x_t, a) \leq \Delta_{t,C_0}\}$;
  \STATE $p_t(a) = \frac{\epsilon}{K} \1\{a \in A_t\} + (1 - \frac{\epsilon |A_t|}{K}) \1\{\pi_t(x_t) = a\}$;
  \STATE {\bfseries return} $p_t$;
\end{algorithmic}
$\verb+learn+(x_t, a_t, \ell_t(a_t), p_t)$:
\begin{algorithmic}
  \STATE $\hat{\ell}_t = \verb+estimator+(x_t, a_t, \ell_t(a_t), p_t(a_t))$;
  \STATE $\hat{c}_t(a) = \begin{cases}
    \hat{\ell}_t(a), &\text{ if }p_t(a) > 0\\
    1, &\text{ otherwise.}
  \end{cases}$
  \STATE $\pi_{t+1} = \verb+csc_oracle+(\pi_t, x_t, \hat{c}_t)$;
\end{algorithmic}
\end{algorithm}

The simplicity of the $\epsilon$-greedy method described in Section~\ref{sub:e_greedy} often makes it
the method of choice for practitioners. However, the uniform exploration over randomly selected actions
can be quite inefficient and costly in practice.
A natural consideration is to restrict this randomization over actions which could plausibly be selected
by the optimal policy $\pi^* = \arg\min_{\pi \in \Pi} L(\pi)$,
where $L(\pi) = \E_{(x, \ell) \sim D}[\ell(\pi(x))]$ is the expected loss of a policy~$\pi$.

To achieve this, we use techniques from disagreement-based active learning~\citep{hanneke2014theory,hsu2010algorithms}.
After observing a context~$x_t$, for any action~$a$, if we can find a policy~$\pi$ that
would choose this action ($\pi(x_t) = a$) instead of the empirically best action~$\pi_t(x_t)$,
while achieving a small loss on past data,
then there is disagreement about how good such an action is, and we allow exploring it.
Otherwise, we are confident that the best policy would not choose this action, thus
we avoid exploring it, and assign it a high cost.
The resulting method is in Algorithm~\ref{alg:egreedy_active}.
Like RegCB, the method requires a known loss range $[c_{min}, c_{max}]$,
and assigns a loss~$c_{max}$ to such unexplored actions
(we consider the range~$[0, 1]$ in Algorithm~\ref{alg:egreedy_active} for simplicity).
The disagreement test we use is based on empirical loss differences,
similar to the Oracular CAL active learning method~\citep{hsu2010algorithms},
denoted \verb+loss_diff+, together with a~threshold:
\[\Delta_{t, C_0} = \sqrt{C_0 \frac{K \log t}{\epsilon t}} + C_0 \frac{K \log t}{\epsilon t}.\]
A practical implementation of \verb+loss_diff+ for an online setting is given below.
We analyze a theoretical form of this algorithm in Section~\ref{sub:active_analysis},
showing a formal version of the following theorem:
\begin{theorem}
\label{thm:regret_informal}
With high-probability, and under favorable conditions on disagreement and on the problem noise,
active $\epsilon$-greedy achieves expected regret $O(T^{1/3})$.
\end{theorem}
Note that this data-dependent guarantee improves on worst-case guarantees achieved by
the optimal algorithms~\citet{agarwal2014taming,dudik2011efficient}.
In the extreme case where the loss of any suboptimal policy is bounded away from that of~$\pi^*$,
we show that our algorithm can achieve constant regret.
While active learning algorithms suggest that data-dependent thresholds~$\Delta_t$ can yield
better guarantees~\citep[\eg,][]{huang2015efficient},
this may require more work in our setting due to open problems
related to data-dependent guarantees for contextual bandits~\citep{agarwal2017open}.
In a worst-case scenario, active $\epsilon$-greedy behaves similarly to $\epsilon$-greedy~\citep{langford2008epoch},
achieving an $O(T^{2/3})$ expected regret with high probability.

\paragraph{Practical implementation of the disagreement test.}
We now present a practical way to implement the disagreement tests in the active $\epsilon$-greedy method,
in the context of online cost-sensitive classification oracles based on regression, as in Vowpal Wabbit.
This corresponds to the $\verb+loss_diff+$ method in Algorithm~\ref{alg:egreedy_active}.

Let~$\hat{L}_{t-1}(\pi)$ denote the empirical loss of policy~$\pi$ on the (biased) sample
of cost-sensitive examples collected up to time $t-1$ (see Section~\ref{sub:active_analysis} for details).
After observing a context $x_t$, we want to estimate
\[
\verb+loss_diff+(\pi_t, x_t, \bar a) \approx \hat{L}_{t-1}(\pi_{t,\bar a}) - \hat{L}_{t-1}(\pi_t),
\]
for any action $\bar a$, where
\begin{align*}
\pi_t &= \arg\min_\pi \hat{L}_{t-1}(\pi) \\
\pi_{t,\bar a} &= \arg\min_{\pi:\pi(x_t)=\bar a} \hat{L}_{t-1}(\pi).
\end{align*}
In our online setup, we take $\pi_t$ to be the current online policy (as in Algorithm~\ref{alg:egreedy_active}),
and we estimate the loss difference by looking at how many online CSC examples of the form $\bar c := (\1\{a \ne \bar{a}\})_{a=1..K}$
are needed (or the importance weight on such an example) in order to switch prediction from $\pi_t(x_t)$ to $\bar a$.
If we denote this importance weight by~$\tau_{\bar a}$,
then we can estimate $\hat L_{t-1}(\pi_{t,\bar{a}}) - \hat L_{t-1}(\pi_t) \approx \tau_{\bar a}/t$.

\paragraph{Computing $\tau_{\bar a}$ for IPS/DR.}
In the case of IPS/DR, we use an online CSC oracle,
which is based on $K$ regressors $f(x, a)$ in VW, each predicting the cost for an action~$a$.
Let $f_t$ be the current regressors for policy $\pi_t$, $y_t(a) := f_t(x_t, a)$,
and denote by $s_t(a)$ the \emph{sensitivity} of regressor $f_t(\cdot, a)$ on example $(x_t, \bar{c}(a))$.
This sensitivity is essentially defined to be the derivative with respect to an importance weight~$w$
of the prediction $y'(a)$ obtained from the regressor after an online update $(x_t, \bar{c}(a))$ with importance weight~$w$.
A similar quantity has been used, \eg, by~\citet{huang2015efficient,karampatziakis2011online,krishnamurthy2017active}.
Then, the predictions on actions $\bar{a}$ and $a$ cross when the importance weight~$w$
satisfies $y_t(\bar{a}) - s_t(\bar{a}) w = y_t(a) + s_t(a) w$.
Thus, the importance weight required for action~$\bar{a}$ to be preferred (\ie, smaller predicted loss)
to action~$a$ is given by:
\[
w_{\bar{a}}^a = \frac{y_t(\bar{a}) - y_t(a)}{s_t(\bar{a}) + s_t(a)}.
\]
Action~$\bar{a}$ will thus be preferred to all other actions when using an importance weight $\tau_{\bar{a}} = \max_a w_{\bar{a}}^a$.

\paragraph{Computing $\tau_{\bar a}$ for \MTR{}.}
Although Algorithm~\ref{alg:egreedy_active} and the theoretical analysis require CSC in order to
assign a loss of 1 to unexplored actions, and hence does not directly support \MTR{},
we can consider an approximation which leverages the benefits of \MTR{} by performing
standard \MTR{} updates as in $\epsilon$-greedy,
while exploring only on actions that pass a similar disagreement test.
In this case, we estimate $\tau_{\bar a}$ as the importance weight on an online regression example $(x_t, 0)$
for the regressor $f_t(\cdot, \bar a)$, needed to switch prediction to $\bar a$.
If $s_t(\bar a)$ is the sensitivity for such an example,
we have $\tau_{\bar a} = (y_t(\bar a) - y_t^*) / s_t(\bar a)$, where $y_t^* = \min_a y_t(a)$.

\subsection{Theoretical Analysis}
\label{sub:active_analysis}

This section presents a theoretical analysis of the active $\epsilon$-greedy
method introduced in Section~\ref{sub:active_impl}.
We begin by presenting the analyzed version of the algorithm together with definitions in Section~\ref{ssub:active_algo}.
Section~\ref{ssub:active_correctness} then studies the correctness of the method,
showing that with high probability, the actions chosen by the optimal policy are always explored,
and that policies considered by the algorithm are always as good as those obtained
under standard $\epsilon$-greedy exploration.
This section also introduces a Massart-type low-noise condition similar to the one considered
by~\citet{krishnamurthy2017active} for cost-sensitive classification.
Finally, Section~\ref{ssub:active_regret} provides a regret analysis of the algorithm,
both in the worst case and under disagreement conditions together with the Massart noise condition.
In particular, a formal version of Theorem~\ref{thm:regret_informal} is given by Theorem~\ref{thm:regret_massart},
and a more extreme but informative situation is considered in Proposition~\ref{prop:policy_gap},
where our algorithm can achieve constant regret.

\subsubsection{Algorithm and definitions}
\label{ssub:active_algo}

We consider a version of the active $\epsilon$-greedy strategy that is more suitable for
theoretical analysis, given in Algorithm~\ref{alg:greedy_active}.
This method considers exact CSC oracles, as well as a CSC oracle with one constraint on
the policy ($\pi(x_t) = a$ in Eq.\eqref{eq:csc_constraint}).
The threshold~$\Delta_t$ is defined later in Section~\ref{ssub:active_correctness}.
Computing it would require some knowledge about the size of the policy class,
which we avoid by introducing a parameter~$C_0$ in the practical variant.
The disagreement strategy is based on the Oracular CAL active learning method of~\citet{hsu2010algorithms},
which tests for disagreement using empirical error differences, and considers biased samples when
no label is queried. Here, similar tests are used to decide which actions should be explored,
in the different context of cost-sensitive classification,
and the unexplored actions are assigned a loss of 1, making the empirical sample biased
($\hat{Z}_T$ in Algorithm~\ref{alg:greedy_active}).

\paragraph{Definitions.}
Define $Z_T = \{(x_t, \ell_t)\}_{t=1..T} \subset \mathcal{X} \times \R^K$, $\tilde{Z}_T = \{(x_t, \tilde{\ell}_t)\}_{t=1..T}$ (biased sample)
and $\hat{Z}_T = \{(x_t, \hat{\ell}_t)\}_{t=1..T}$ (IPS estimate of biased sample),
where $\ell_t \in [0, 1]^K$ is the (unobserved) loss vector at time~$t$ and
\begin{align}
\tilde{\ell}_t(a) &= \begin{cases}
	\ell_t(a), &\text{ if }a \in A_t\\
	1, &\text{ o/w}
\end{cases} \\
\hat{\ell}_t(a) &= \begin{cases}
	\frac{\1\{a = a_t\}}{p_t(a_t)} \ell_t(a_t), &\text{ if }a \in A_t\\
	1, &\text{ o/w.}
\end{cases} \label{eq:ell_hat}
\end{align}

For any set $Z \subset \mathcal{X} \times \R^K$ defined as above, we denote, for $\pi \in \Pi$,
\begin{align*}
L(\pi, Z) = \frac{1}{|Z|} \sum_{(x, c) \in Z} c(\pi(x)).
\end{align*}
We then define the empirical losses $L_T(\pi) := L(\pi, Z_T)$, $\hat{L}_T(\pi) := L(\pi, \hat{Z}_T)$ and $\tilde{L}_T(\pi) := L(\pi, \tilde{Z}_T)$.
Let $L(\pi) := \E_{(x, \ell) \sim D} [ \ell(\pi(x)) ]$ be the expected loss of policy $\pi$,
and $\pi^* := \arg\min_{\pi \in \Pi} L(\pi)$.
We also define $\rho(\pi, \pi') := P_x(\pi(x) \ne \pi'(x))$, the expected disagreement between policies $\pi$ and $\pi'$,
where $P_x$ denotes the marginal distribution of $D$ on contexts.

\begin{algorithm}[tb]
\caption{active $\epsilon$-greedy: analyzed version}
\label{alg:greedy_active}
\begin{algorithmic}
\STATE {\bfseries Input:} exploration probability $\epsilon$.
\STATE Initialize: $\hat{Z}_0 := \emptyset$.
\FOR{$t = 1, \ldots$}
  \STATE Observe context $x_t$. Let
  \begin{align}
  \pi_t &:= \arg\min_{\pi} L(\pi, \hat{Z}_{t-1}) \nonumber \\
  \pi_{t,a} &:= \arg\min_{\pi : \pi(x_t) = a} L(\pi, \hat{Z}_{t-1}) \label{eq:csc_constraint} \\
  A_t &:= \{a : L(\pi_{t,a}, \hat{Z}_{t-1}) - L(\pi_{t}, \hat{Z}_{t-1}) \leq \Delta_t \} \label{eq:at_def}
  \end{align}
  \STATE Let
  \[
  p_t(a) = \begin{cases}
  	1 - (|A_t| - 1) \epsilon / K, &\text{ if }a = \pi_t(x_t)\\
  	\epsilon / K, &\text{ if }a \in A_t \setminus \{\pi_t(x_t)\} \\
  	0, &\text{ otherwise.}
  \end{cases}
  \]
  \STATE Play action $a_t \sim p_t$, observe $\ell_t(a_t)$ and set $\hat{Z}_t = \hat{Z}_{t-1} \cup \{(x_t, \hat{\ell}_t)\}$, where $\hat{\ell}_t$ is defined in~\eqref{eq:ell_hat}.
\ENDFOR
\end{algorithmic}
\end{algorithm}

\subsubsection{Correctness}
\label{ssub:active_correctness}

We begin by stating a lemma that controls deviations of empirical loss differences,
which relies on Freedman's inequality for martingales~\citep[see, \eg,][Lemma 3]{kakade2009generalization}.
\begin{lemma}[Deviation bounds]
\label{lemma:deviations}
With probability $1 - \delta$, the following event holds: for all $\pi \in \Pi$, for all $T \geq 1$,
\begin{align}
|(\hat{L}_T(\pi) - \hat{L}_T(\pi^*)) - (\tilde{L}_T(\pi) - \tilde{L}_T(\pi^*))| &\leq \sqrt{\frac{2K \rho(\pi, \pi^*) e_T}{\epsilon}} + \left(\frac{K}{\epsilon} + 1 \right) e_T \label{eq:tilde_bound} \\
|(L_T(\pi) - L_T(\pi^*)) - (L(\pi) - L(\pi^*))| &\leq \sqrt{\rho(\pi, \pi^*) e_T} + 2e_T, \label{eq:sample_bound}
\end{align}
where $e_T = \log (2|\Pi| / \delta_T) / T$ and $\delta_T = \delta / (T^2 + T)$.
We denote this event by~$\mathcal{E}$ in what follows.
\end{lemma}

\begin{proof}
We prove the result using Freedman's inequality~\citep[see, \eg,][Lemma 3]{kakade2009generalization},
which controls deviations of a sum using the conditional variance of each term in the sum
and an almost sure bound on their magnitude, along with a union bound.

For~\eqref{eq:tilde_bound}, let $(\hat{L}_T(\pi) - \hat{L}_T(\pi^*)) - (\tilde{L}_T(\pi) - \tilde{L}_T(\pi^*)) = \frac{1}{T}\sum_{t=1}^T R_t$, with
\begin{align*}
R_t = \hat{\ell}_t(\pi(x_t)) - \hat{\ell}_t(\pi^*(x_t)) - (\tilde{\ell}_t(\pi(x_t)) - \tilde{\ell}_t(\pi^*(x_t))).
\end{align*}
We define the $\sigma$-fields $\mathcal{F}_t := \sigma(\{x_i, \ell_i, a_i\}_{i=1}^t)$. Note that $R_t$ is $\mathcal{F}_t$-measurable and
\begin{align*}
\E[\hat{\ell}_t(\pi(x_t)) - \hat{\ell}_t(\pi^*(x_t)) | x_t, \ell_t] = \tilde{\ell}_t(\pi(x_t)) - \tilde{\ell}_t(\pi^*(x_t)),
\end{align*}
so that $\E[R_t | \mathcal{F}_{t-1}] = \E[\E[R_t |x_t, \ell_t] | \mathcal{F}_{t-1}] = 0$. Thus, $(R_t)_{t\geq 1}$ is a martingale difference sequence adapted to the filtration $(\mathcal{F}_t)_{t\geq 1}$.
We have
\begin{align*}
|R_t| \leq |\hat{\ell}_t(\pi(x_t)) - \hat{\ell}_t(\pi^*(x_t))| + |\tilde{\ell}_t(\pi(x_t)) - \tilde{\ell}_t(\pi^*(x_t))| \leq \frac{K}{\epsilon} + 1.
\end{align*}
Note that $\E[\hat{\ell}_t(\pi(x_t)) - \hat{\ell}_t(\pi^*(x_t)) | x_t, \ell_t] = \tilde{\ell}_t(\pi(x_t)) - \tilde{\ell}_t(\pi^*(x_t))$, so that
\begin{align*}
\E[R_t^2 | \mathcal{F}_{t-1}] &= \E[\E[R_t^2 | x_t, \ell_t, A_t] | \mathcal{F}_{t-1}] \\
	&\leq \E[ \E[(\hat{\ell}_t(\pi(x_t)) - \hat{\ell}_t(\pi^*(x_t)))^2| x_t, \ell_t, A_t] | \mathcal{F}_{t-1}] \\
	&\leq \E \left[\E \left[\frac{(\1\{\pi(x_t) = a_t\} - \1\{\pi^*(x_t) = a_t\})^2}{p_t(a_t)^2} | x_t, \ell_t, A_t \right] | \mathcal{F}_{t-1} \right] \\
	&\leq \E \left[\E \left[\frac{\1\{\pi(x_t) \ne \pi^*(x_t)\}(\1\{\pi(x_t) = a_t\} + \1\{\pi^*(x_t) = a_t\})}{p_t(a_t)^2} | x_t, \ell_t, A_t \right] | \mathcal{F}_{t-1} \right] \\
	&= \E\left[\frac{2 K \1\{\pi(x_t) \ne \pi^*(x_t)\}}{\epsilon} |\mathcal{F}_{t-1} \right] = \frac{2K}{\epsilon} \rho(\pi, \pi^*).
\end{align*}
Freedman's inequality then states that~\eqref{eq:tilde_bound} holds with probability $1 - \delta_T/2 |\Pi|$.

For~\eqref{eq:sample_bound}, we consider a similar setup with
\begin{align*}
R_t = \ell_t(\pi(x_t)) - \ell_t(\pi^*(x_t)) - (L(\pi) - L(\pi^*)).
\end{align*}
We have $\E[R_t|\mathcal{F}_{t-1}] = 0$, $|R_t| \leq 2$ and $\E[R_t^2|\mathcal{F}_{t-1}] \leq \rho(\pi, \pi^*)$,
which yields that $\eqref{eq:sample_bound}$ holds with probability $1 - \delta_T/2 |\Pi|$ using Freedman's inequality.
A union bound on $\pi \in \Pi$ and $T \geq 1$ gives the desired result.
\end{proof}

\paragraph{Threshold.}
We define the threshold $\Delta_T$ used in~\eqref{eq:at_def} in Algorithm~\ref{alg:greedy_active} as:
\begin{equation}
\label{eq:threshold}
\Delta_T := \left(\sqrt{\frac{2K}{\epsilon}} + 1\right) \sqrt{e_{T-1}} + \left( \frac{K}{\epsilon} + 3\right) e_{T-1}.
\end{equation}
We also define the following more precise deviation quantity for a given policy,
which follows directly from the deviation bounds in Lemma~\ref{lemma:deviations}
\begin{equation}
\label{eq:deviation_tot}
\Delta_T^*(\pi) := \left(\sqrt{\frac{2K}{\epsilon}} + 1\right) \sqrt{\rho(\pi, \pi^*) e_{T-1}} + \left( \frac{K}{\epsilon} + 3\right) e_{T-1}.
\end{equation}
Note that we have $\Delta_T^*(\pi) \leq \Delta_T$ for any policy $\pi$.

The next lemma shows that the bias introduced in the empirical sample by assigning a loss of 1
to unexplored actions is favorable, in the sense that it will not hurt us in identifying~$\pi^*$.
\begin{lemma}[Favorable bias]
\label{lemma:favorable_bias}
Assume $\pi^*(x_t) \in A_t$ for all $t \leq T$. We have
\begin{equation}
\label{eq:favorable_bias}
\tilde{L}_{T}(\pi) - \tilde{L}_{T}(\pi^*) \geq L_{T}(\pi) - L_{T}(\pi^*).
\end{equation}
\end{lemma}
\begin{proof}
For any $t \leq T$, we have $\tilde{\ell}_t(a) \geq \ell_t(a)$, so that $\tilde{L}_{T}(\pi) \geq L_{T}(\pi)$.
Separately, we have $\tilde{\ell}_t(\pi^*(x_t)) = \ell_t(\pi^*(x_t))$ for all $t \leq T$ using the definition of $\tilde{\ell}_t$ and the assumption $\pi^*(x_t) \in A_t$,
hence $\tilde{L}_{T}(\pi^*) \geq L_{T}(\pi^*)$.
\end{proof}

We now show that with high probability, the optimal action is always explored by the algorithm.
\begin{lemma}
\label{lemma:pi_star}
Assume that event~$\mathcal{E}$ holds.
The actions given by the optimal policy are always explored for all $t \geq 1$,
i.e., $\pi^*(x_t) \in A_t$ for all $t \geq 1$.
\end{lemma}
\begin{proof}
We show by induction on $T \geq 1$ that $\pi^*(x_t) \in A_t$ for all $t = 1, \ldots, T$.
For the base case, we have $A_1 = [K]$ since $\hat{Z}_0 = \emptyset$ and hence empirical errors
are always equal to 0, so that $\pi^*(x_1) \in A_1$.
Let us now assume as the inductive hypothesis that $\pi^*(x_t) \in A_t$ for all $t \leq T-1$.

From deviation bounds, we have
\begin{align*}
\hat{L}_{T-1}(\pi_T) - \hat{L}_{T-1}(\pi^*) &\geq \tilde{L}_{T-1}(\pi_T) - \tilde{L}_{T-1}(\pi^*) - \left(\sqrt{\frac{2K \rho(\pi, \pi^*)e_{T-1}}{\epsilon}} + (K/\epsilon + 1) e_{T-1}\right) \\
L_{T-1}(\pi_T) - L_{T-1}(\pi^*) &\geq L(\pi_T) - L(\pi^*) - \left(\sqrt{\rho(\pi, \pi^*)e_{T-1}} + 2e_{T-1}\right).
\end{align*}
Using Lemma~\ref{lemma:favorable_bias} together with the inductive hypothesis, the above inequalities yield
\begin{align*}
\hat{L}_{T-1}(\pi_T) - \hat{L}_{T-1}(\pi^*) \geq L(\pi_T) - L(\pi^*) - \Delta_T^*(\pi_T).
\end{align*}

Now consider an action $a \notin A_t$. Using the definition~\eqref{eq:at_def} of $A_t$, we have
\begin{align*}
\hat{L}_{T-1}(\pi_{T,a}) - \hat{L}_{T-1}(\pi^*) &= \hat{L}_{T-1}(\pi_{T,a}) - \hat{L}_{T-1}(\pi_{T}) + \hat{L}_{T-1}(\pi_{T}) - \hat{L}_{T-1}(\pi^*) \\
	&> \Delta_T - \Delta_T^*(\pi_T) = 0,
\end{align*}
which implies $\pi^*(x_T) \ne a$, since $\hat{L}_{T-1}(\pi_{T,a})$ is the minimum of $\hat{L}_{T-1}$ over policies satisfying $\pi(x_T)=a$.
This yields $\pi^*(x_T) \in A_T$, which concludes the proof.
\end{proof}

With the previous results, we can now prove that with high probability,
discarding some of the actions from the exploration process does not hurt us
in identifying good policies.
In particular, $\pi_{T+1}$ is about as good as it would have been with
uniform $\epsilon$-exploration all along.
\begin{theorem}
\label{thm:correctness}
Under the event $\mathcal{E}$, which holds with probability $1 - \delta$,
\begin{equation*}
L(\pi_{T+1}) - L(\pi^*) \leq \Delta_{T+1}^*(\pi_{T+1}).
\end{equation*}
In particular, $L(\pi_{T+1}) - L(\pi^*) \leq \Delta_{T+1}$.
\end{theorem}
\begin{proof}
Assume event $\mathcal{E}$ holds.
Using~(\ref{eq:tilde_bound}-\ref{eq:sample_bound}) combined with Lemma~\ref{lemma:favorable_bias} (which holds by Lemma~\ref{lemma:pi_star}), we have
\begin{align*}
L(\pi_{T+1}) - L(\pi^*) \leq \hat{L}_T(\pi_{T+1}) - \hat{L}_T(\pi^*) + \Delta_{T+1}^*(\pi_{T+1}) \leq \Delta_{T+1}^*(\pi_{T+1}).
\end{align*}
\end{proof}

\paragraph{Massart noise condition.}
We introduce a low-noise condition that will help us obtain improved regret guarantees.
Similar conditions have been frequently used in supervised learning~\citep{massart2006risk}
and active learning~\citep{hsu2010algorithms,huang2015efficient,krishnamurthy2017active}
for obtaining better data-dependent guarantees.
We consider the following Massart noise condition with parameter $\tau > 0$:
\begin{equation}\tag{M}
\label{eq:massart}
\rho(\pi, \pi^*) \leq \frac{1}{\tau} (L(\pi) - L(\pi^*)).
\end{equation}
This condition holds when $\E[\min_{a \ne \pi^*(x)} \ell(a) - \ell(\pi^*(x)) | x] \geq \tau$, $P_x$-almost surely,
which is similar to the Massart condition considered in~\cite{krishnamurthy2017active}
in the context of active learning for cost-sensitive classification.
Indeed, we have
\begin{align*}
L(\pi) - L(\pi^*) &= \E[\1\{\pi(x) \ne \pi^*(x)\} (\ell(\pi(x)) - \ell(\pi^*(x))] \\
  &\qquad + \E[\1\{\pi(x) = \pi^*(x)\} (\ell(\pi^*(x)) - \ell(\pi^*(x)))] \\
 	&\geq \E \left[\1\{\pi(x) \ne \pi^*(x)\} \left(\min_{a \ne \pi^*(x)} \ell(a) - \ell(\pi^*(x)) \right) \right] \\
 	&= \E[\1\{\pi(x) \ne \pi^*(x)\} \E[\min_{a \ne \pi^*(x)} \ell(a) - \ell(\pi^*(x))|x]] \\
	&\geq \E[\1\{\pi(x) \ne \pi^*(x)\} \tau] = \tau \rho(\pi, \pi^*),
\end{align*}
which is precisely~\eqref{eq:massart}.
The condition allows us to obtain a fast rate for the policies considered by our algorithm, as we now show.

\begin{theorem}
Assume the Massart condition~\eqref{eq:massart} holds with parameter~$\tau$.
Under the event $\mathcal{E}$, which holds w.p. $1 - \delta$,
\begin{equation*}
L(\pi_{T+1}) - L(\pi^*) \leq C \frac{K}{\tau \epsilon} e_T,
\end{equation*}
for some numeric constant~$C$.
\end{theorem}
\begin{proof}
Using Theorem~\ref{thm:correctness} and the Massart condition, we have
\begin{align*}
L(\pi_{T+1}) - L(\pi^*) &\leq \Delta_{T+1}^*(\pi_{T+1}) = \left(\sqrt{\frac{2K}{\epsilon}} + 1\right) \sqrt{\rho(\pi_{T+1}, \pi^*) e_{T}} + \left( \frac{K}{\epsilon} + 3\right) e_{T} \\
	&\leq \left(\sqrt{\frac{2K}{\epsilon}} + 1\right) \sqrt{(L(\pi_{T+1}) - L(\pi^*)) e_{T} / \tau} + \left( \frac{K}{\epsilon} + 3\right) e_{T} \\
	&\leq \sqrt{\frac{8K e_{T}}{\tau \epsilon}(L(\pi_{T+1}) - L(\pi^*))} + \frac{4 K e_{T}}{\epsilon}.
\end{align*}
Solving the quadratic inequality in $L(\pi_{T+1}) - L(\pi^*)$ yields the result.
\end{proof}

\subsubsection{Regret Analysis}
\label{ssub:active_regret}

In a worst-case scenario, the following result shows that Algorithm~\ref{alg:greedy_active}
enjoys a similar~$O(T^{2/3})$ regret guarantee to the vanilla $\epsilon$-greedy approach~\citep{langford2008epoch}.
\begin{theorem}
Conditioned on the event~$\mathcal{E}$, which holds with probability $1- \delta$,
the expected regret of the algorithm is
\begin{equation*}
\E[R_T | \mathcal{E}] \leq O \left(\sqrt{\frac{KT \log(T|\Pi|/\delta)}{\epsilon}} + T \epsilon \right).
\end{equation*}
Optimizing over the choice of $\epsilon$ yields a regret $O(T^{2/3} (K \log(T|\Pi|/\delta))^{1/3})$.
\end{theorem}
\begin{proof}
We condition on the $1- \delta$ probability event~$\mathcal{E}$ that the deviation bounds of Lemma~\ref{lemma:deviations} hold.
We have
\begin{align*}
\E[\ell_t(a_t) - \ell_t(\pi^*(x_t)) | \mathcal{F}_{t-1}] &=  \E[\1\{a_t = \pi_t(x_t)\} (\ell_t(\pi_t(x_t)) - \ell_t(\pi^*(x_t)))| \mathcal{F}_{t-1}] \\
  &\qquad + \E[\1\{a_t \ne \pi_t(x_t)\}(\ell_t(a_t) - \ell_t(\pi^*(x_t))) | \mathcal{F}_{t-1}] \\
	&\leq \E[\ell_t(\pi_t(x_t)) - \ell_t(\pi^*(x_t)) | \mathcal{F}_{t-1}] + \E[\E[1 - p_t(\pi_t(x_t)) | x_t] | \mathcal{F}_{t-1}] \\
	&\leq L(\pi_t) - L(\pi^*) + \epsilon.
\end{align*}
Summing over~$t$ and applying Theorem~\ref{thm:correctness} together with $\Delta_t^*(\pi) \leq \Delta_t$, we obtain
\begin{align*}
\E[R_T | \mathcal{E}] &= \E \left[ \sum_{t=1}^T \ell_t(a_t) - \ell_t(\pi^*(x_t)) | \mathcal{E} \right] \\
	&\leq 1 + \sum_{t=2}^T \E[L(\pi_t) - L(\pi^*) + \epsilon | \mathcal{F}_{t-1}, \mathcal{E}] \\
	&\leq 1 + T \epsilon + \sum_{t=2}^T \Delta_t.
\end{align*}
Using $\sum_{t=2}^T \sqrt{e_t} \leq O(\sqrt{T \log(8 T^2 |\Pi|/\delta)})$
and $\sum_{t=2}^T e_t \leq O(\log(8 T^2 |\Pi|/\delta) \log T )$, we obtain
\begin{equation*}
\E[R_T | \mathcal{E}] \leq O \left(1 + \sqrt{\frac{KT \log(T|\Pi|/\delta)}{\epsilon}} + \frac{K\log(T|\Pi|/\delta)}{\epsilon} \log T + T \epsilon \right),
\end{equation*}
which yields the result.
\end{proof}

\paragraph{Disagreement definitions.}
In order to obtain improvements in regret guarantees over the worst case,
we consider notions of disagreement that extend standard definitions from the active learning
literature~\cite[\eg,][]{hanneke2014theory,hsu2010algorithms,huang2015efficient} to the multiclass case.
Let $B(\pi^*, r) := \{\pi \in \Pi : \rho(\pi, \pi^*) \leq r\}$ be the ball centered at~$\pi^*$ under
the (pseudo)-metric $\rho(\cdot, \cdot)$.
We define the disagreement region~$DIS(r)$ and disagreement coefficient~$\theta$ as follows:
\begin{align*}
DIS(r) &:= \{x : \exists \pi \in B(\pi^*, r) \quad \pi(x) \ne \pi^*(x)\} \\
\theta &:= \sup_{r>0} \frac{P(x \in DIS(r))}{r}.
\end{align*}

The next result shows that under the Massart condition and with a finite disagreement coefficient~$\theta$,
our algorithm achieves a regret that scales as $O(T^{1/3})$ (up to logarithmic factors),
thus improving on worst-case guarantees obtained by optimal algorithms such as~\citet{agarwal2012contextual,agarwal2014taming,dudik2011efficient}.

\begin{theorem}
\label{thm:regret_massart}
Assume the Massart condition~\eqref{eq:massart} holds with parameter~$\tau$.
Conditioning on the event~$\mathcal{E}$ which holds w.p. $1 - \delta$, the algorithm has expected regret
\begin{equation*}
\E[R_T|\mathcal{E}] \leq O \left(\frac{K \log(T|\Pi|/\delta)}{\tau \epsilon} \log T + \frac{\theta}{\tau} \sqrt{\epsilon KT \log(T|\Pi|/\delta)} \right).
\end{equation*}
Optimizing over the choice of $\epsilon$ yields a regret
\begin{equation*}
\E[R_T|\mathcal{E}] \leq O \left( \frac{1}{\tau} (\theta K \log(T|\Pi|/\delta))^{2/3} (T \log T)^{1/3} \right).
\end{equation*}
\end{theorem}

\begin{proof}
Assume $\mathcal{E}$ holds. Let $t \geq 2$, and assume $a \in A_t \setminus \{\pi^*(x_t)\}$.
Define
\begin{equation*}
\pi_a = \begin{cases}
	\pi_t, &\text{ if }\pi_t(x_t) = a\\
	\pi_{t,a}, &\text{ if }\pi_t(x_t) \ne a,
\end{cases}
\end{equation*}
so that we have $\pi_a(x_t) = a \ne \pi^*(x_t)$.
\begin{itemize}
	\item If $\pi_a = \pi_t$, then $L(\pi_a) - L(\pi^*) \leq \Delta_t^*(\pi_a) \leq \Delta_t$ by Theorem~\ref{thm:correctness}
	\item If $\pi_a = \pi_{t,a}$, using deviation bounds, Lemma~\ref{lemma:pi_star} and~\ref{lemma:favorable_bias}, we have
	\begin{align*}
	L(\pi_a) - L(\pi^*) &= L(\pi_{t,a}) - L(\pi^*) \\
		&\leq \hat{L}_{t-1}(\pi_{t,a}) - \hat{L}_{t-1}(\pi^*) + \Delta_t^*(\pi_{t,a}) \\
		&= \underbrace{\hat{L}_{t-1}(\pi_{t,a}) - \hat{L}_{t-1}(\pi_t)}_{\leq \Delta_t} + \underbrace{\hat{L}_{t-1}(\pi_t) - \hat{L}_{t-1}(\pi^*)}_{\leq 0} + \Delta_t^*(\pi_{t,a}) \\
		&\leq 2 \Delta_t,
	\end{align*}
	where the last inequality uses $a \in A_t$.
\end{itemize}
By the Massart assumption, we then have $\rho(\pi_a, \pi^*) \leq 2 \Delta_t / \tau$.
Hence, we have $x_t \in DIS(2 \Delta_t / \tau)$. We have thus shown
\begin{equation*}
\E [\E[\1\{a \in A_t \setminus \{\pi^*(x_t)\}\} | x_t] | \mathcal{F}_{t-1}] \leq \E[P(x_t \in DIS(2 \Delta_t / \tau)) | \mathcal{F}_{t-1}] \leq 2 \theta \Delta_t / \tau.
\end{equation*}
We then have
\begin{align*}
\E[\ell_t(a_t) - \ell_t(\pi^*(x_t)) | \mathcal{F}_{t-1}]
	&= \E [\1\{a_t = \pi_t(x_t)\} (\ell_t(\pi_t(x_t)) - \ell_t(\pi^*(x_t))) \\
	&\quad + \1\{a_t = \pi^*(x_t) \wedge a_t \ne \pi_t(x_t)\} (\ell_t(\pi^*(x_t)) - \ell_t(\pi^*(x_t))) \\
	&\quad + \sum_{a=1}^K \1\{a_t = a \wedge a \notin \{\pi_t(x_t), \pi^*(x_t)\}\} (\ell_t(a) - \ell_t(\pi^*(x_t))) | \mathcal{F}_{t-1}] \\
	&\leq \E [ \ell_t(\pi_t(x_t)) - \ell_t(\pi^*(x_t)) | \mathcal{F}_{t-1}] \\
  &\quad + \E \left[\sum_{a=1}^K \E[\1\{a_t = a \} \1\{a \notin \{\pi_t(x_t), \pi^*(x_t)\}\}|x_t] | \mathcal{F}_{t-1} \right] \\
	&= L(\pi_t) - L(\pi^*) + \sum_{a=1}^K \E[ \E[ p_t(a) \1\{a \notin \{\pi_t(x_t), \pi^*(x_t)\}\} | x_t] | \mathcal{F}_{t-1}] \\
	&\leq L(\pi_t) - L(\pi^*) + \frac{\epsilon}{K} \sum_{a=1}^K \E[\E[\1\{a \in A_t \setminus\{\pi^*(x_t)\}\} | x_t] | \mathcal{F}_{t-1}] \\
	&\leq C \frac{K}{\tau \epsilon} e_{t-1} + 2 \epsilon \theta \Delta_t / \tau,
\end{align*}
where we used
\begin{align*}
p_t(a) \1\{a \notin \{\pi_t(x_t), \pi^*(x_t)\}\} &= \frac{\epsilon}{K} \1\{a \in A_t \setminus \{\pi_t(x_t), \pi^*(x_t)\}\} \\
	&\leq \frac{\epsilon}{K} \1\{a \in A_t \setminus \{\pi^*(x_t)\}\}.
\end{align*}
Summing over $t$ and taking total expectations (conditioned on $\mathcal{E}$) yields
\begin{equation*}
\E[R_T|\mathcal{E}] \leq O \left(\frac{K \log(T|\Pi|/\delta)}{\tau \epsilon} \log T + \frac{\epsilon \theta}{\tau} \left(\sqrt{\frac{KT \log(T|\Pi|/\delta)}{\epsilon}} + \frac{K\log(T|\Pi|/\delta)}{\epsilon} \log(T) \right) \right),
\end{equation*}
and the result follows.
\end{proof}

Finally, we look at a simpler instructive example, which considers an extreme situation where the expected loss
of any suboptimal policy is bounded away from that of the optimal policy.
In this case, Algorithm~\ref{alg:greedy_active} can achieve constant regret when the disagreement coefficient is bounded,
as shown by the following result.
\begin{proposition}
\label{prop:policy_gap}
Assume that $L(\pi) - L(\pi^*) \geq \tau > 0$ for all $\pi \ne \pi^*$, and that $\theta < \infty$.
Under the event~$\mathcal{E}$, the algorithm achieves constant expected regret.
In particular, the algorithm stops incurring regret for $T > T_0 := \max \{t : 2 \Delta_t > \tau \}$.
\end{proposition}
\begin{proof}
By Theorem~\ref{thm:correctness} and our assumption, we have $L(\pi_t) - L(\pi^*) \leq \1\{\Delta_t \geq \tau\} \Delta_t$.
Similarly, the assumption implies that $\rho(\pi, \pi^*) \leq \1\{L(\pi) - L(\pi^*) \geq \tau\}$,
so that using similar arguments to the proof of Theorem~\ref{thm:regret_massart}, we have
\begin{equation*}
\E [\E[\1\{a \in A_t \setminus \{\pi^*(x_t)\}\} | x_t] | \mathcal{F}_{t-1}] \leq \theta \1\{2 \Delta_t \geq \tau\}.
\end{equation*}
Following the proof of Theorem~\ref{thm:regret_massart}, this implies that when~$t$ is such that $2 \Delta_t < \tau$, then we have
\begin{align*}
\E[\ell_t(a_t) - \ell_t(\pi^*(x_t)) | \mathcal{F}_{t-1}] = 0.
\end{align*}
Let $T_0 := \max \{t : 2 \Delta_t \geq \tau \}$. We thus have
\begin{align*}
\E[R_T | \mathcal{E}] \leq 1 + \sum_{t=2}^{T_0} (\Delta_t + \epsilon).
\end{align*}
\end{proof}

\end{document}